\documentclass[11pt]{article}

\usepackage{amsfonts,amsthm,amsmath,amssymb,mathtools}
\usepackage{fullpage}
\usepackage{thmtools}

\usepackage[sort&compress]{natbib}
\usepackage{algorithm,algorithmicx}
\usepackage{subcaption}
\usepackage{graphicx}
\usepackage{color}
\usepackage[export]{adjustbox}
\usepackage[outdir=./]{epstopdf}
\usepackage{nicefrac}       

\usepackage{xcolor}
\colorlet{blue}{black}

\def\rc#1{\color{black}#1}

\def\rrc{\color{black}}
\def\bbc{\color{black}}

\allowdisplaybreaks[4]

\usepackage{url}
\usepackage{algcompatible}
\usepackage{algpseudocode}

\def\proj{\text{P}}

\def\1{{\bf{1}}}
\def\0{{\bf{0}}}

\def\w{{\bf w}}

\def\v{{\bf v}}

\def\x{{\bf x}}
\def\y{{\bf y}}
\def\z{{\bf z}}

\def\g{{\bf g}}

\def\s{{\bf s}}
\def\r{{\bf r}}
\def\e{{\bf e}}
\def\W{{W}}

\newcommand{\R}{\mathbb{R}}

\newcommand{\E}{\mathbb{E}}
\let\P\undefined
\newcommand{\P}{\mathbb{P}}
\newcommand{\supp}[1]{\text{supp}(#1)}
\DeclareMathOperator*{\argmin}{arg\,min}

\usepackage[figuresright]{rotating}
\usepackage[ruled,vlined,algo2e]{algorithm2e}

\newcounter{thm_counter}
\newcounter{def_counter}
\newcounter{pro_counter}
\newtheorem{theorem}[thm_counter]{Theorem}
\newtheorem{lemma}[thm_counter]{Lemma}

\newtheorem{definition}[def_counter]{Definition}
\newtheorem{proposition}[pro_counter]{Proposition}

\title{On The Projection Operator to A Three-view Cardinality Constrained Set}

\author{Haichuan Yang\textsuperscript{*}, Shupeng Gui\textsuperscript{*}, Chuyang Ke\textsuperscript{*}\\
Daniel  Stefankovic\textsuperscript{*}, Ryohei  Fujimaki\textsuperscript{\textdagger}, and Ji Liu\textsuperscript{*}\\
  \textsuperscript{*}Department of Computer Science, University of Rochester\\
  \textsuperscript{\textdagger}Knowledge Discovery Research Laboratories, NEC Corporation\\
  \texttt{h.yang@rochester.edu, shupenggui@gmail.com, cke@u.rochester.edu}\\
  \texttt{stefanko@cs.rochester.edu, rfujimaki@nec-labs.com, ji.liu.uwisc@gmail.com}}
\date{\today}

\begin{document}
\maketitle

%

%

\begin{abstract} 
The cardinality constraint is an intrinsic way to restrict the solution structure in many domains, for example, sparse learning, feature selection, and compressed sensing. To solve a cardinality constrained problem, the key challenge is to solve the projection onto the cardinality constraint set, which is NP-hard in general when there exist multiple overlapped cardinality constraints. In this paper, we consider the scenario where the overlapped cardinality constraints satisfy a Three-view Cardinality Structure (TVCS), which reflects the natural restriction in many applications, such as identification of gene regulatory networks and task-worker assignment problem. We cast the projection into a linear programming, and show that for TVCS, the vertex solution of this linear programming is the solution for the original projection problem. We further prove that such solution can be found with the complexity proportional to the number of variables and constraints. We finally use synthetic experiments and two interesting applications in bioinformatics and crowdsourcing to validate the proposed TVCS model and method.
\end{abstract}

\section{Introduction}

The cardinality constraint is an intrinsic way to restrict the solution structure in many real problems, for example, sparse learning \citep{olshausen1997sparse}, feature selection \citep{zhang2009consistency}, and compressed sensing \citep{candes2006stable}. The generic cardinality constrained optimization can be expressed as
\begin{subequations}
\begin{align}
\min_{\w\in \R^p} \quad &f(\w) \label{eq:obj}\\
\text{subject to } \quad &\|\w_{g}\|_0 \leq \s^g\quad \forall g\in \mathcal{G}\label{eq:gsparse}
\end{align}
\label{eq:main}
\end{subequations}

where $\w$ is the optimization variable, $g$ is an index subset of $[p]:=\{1,2,\cdots, p\}$, $\w_g$ is the sub-vector of $\w$ indexed by $g$. $\|\w_{g}\|_0$ denotes the cardinality of the sub-vector, i.e., the number of nonzeros in $\w_{g}$, $\mathcal{G}$ is the hyper set of all predefined groups, and $\s \in \R^{|\mathcal{G}|}$ is the upper bound vector - $\s^g\in \R$ refers to the upper bound of the sparsity over group $g$. Objective $f$ is the loss function which could be defined with different form according to the specific application. The problem \eqref{eq:main} refers to a nonconvex optimization (NP-hard) due to the cardinality constraint. Some efficient iterative methods such as IHT \citep{yuan2014gradient}, CoSaMP \citep{needell2009cosamp}, GradMP \citep{gradmp}, and their variants can guarantee to solve the original problem under some mild conditions. A key component in all of these methods is the projection operator
\begin{align}
\text{P}_{\Omega(\mathcal{G}, \s)}(\v) := \argmin_{\w \in \Omega(\mathcal{G}, \s)}\quad & \|\w - \v\|^2 \label{eq:proj_obj}
\end{align}
where $\Omega(\mathcal{G}, \s)$ denotes the feasible set to the constraint~\eqref{eq:gsparse}. While in some special case, for example, $\mathcal{G}= \{[p]\}$, the projection is trivial, it is quite challenging, especially when $\mathcal{G}$ includes multiple \emph{overlapped} index sets (even NP-hard in some cases). 

In this paper, we consider the scenario where the overlapped cardinality constraints~\eqref{eq:gsparse} satisfy a Three-view Cardinality Structure (TVCS):

\begin{definition}\label{def:overlap}
{\bf (Three-view Cardinality Structure (TVCS))}
For $\Omega(\mathcal{G}, \s)$, the hyper set $\mathcal{G}$ consisting of subsets of $[p]$ admits the TVCS structure if the following conditions are satisfied:
\begin{itemize}
\item There exists a partition $\mathcal{G}_0$, $\mathcal{G}_1$ and $\mathcal{G}_2$ such that $\mathcal{G} =   \mathcal{G}_0 \cup \mathcal{G}_1 \cup \mathcal{G}_2$;
\item $\mathcal{G}_0=\{[p]\}$;
\item All element sets in $\mathcal{G}_1$ have no overlap;
\item All element sets in $\mathcal{G}_2$ have no overlap.
\end{itemize}
\end{definition}

This definition basically requires that $\mathcal{G}$ can be partitioned into three hyper sets $\mathcal{G}_0$, $\mathcal{G}_1$, and $\mathcal{G}_2$, and overlaps can only happen between element sets in different hyper sets. 
$\mathcal{G}_0$ is usually used to restrict the overall sparsity. 
Figure~\ref{fig:tvcs} provides two examples of $\mathcal{G}$ for TVCS. 
\begin{figure}[tbp]
\centering
\includegraphics[width = 0.45\linewidth]{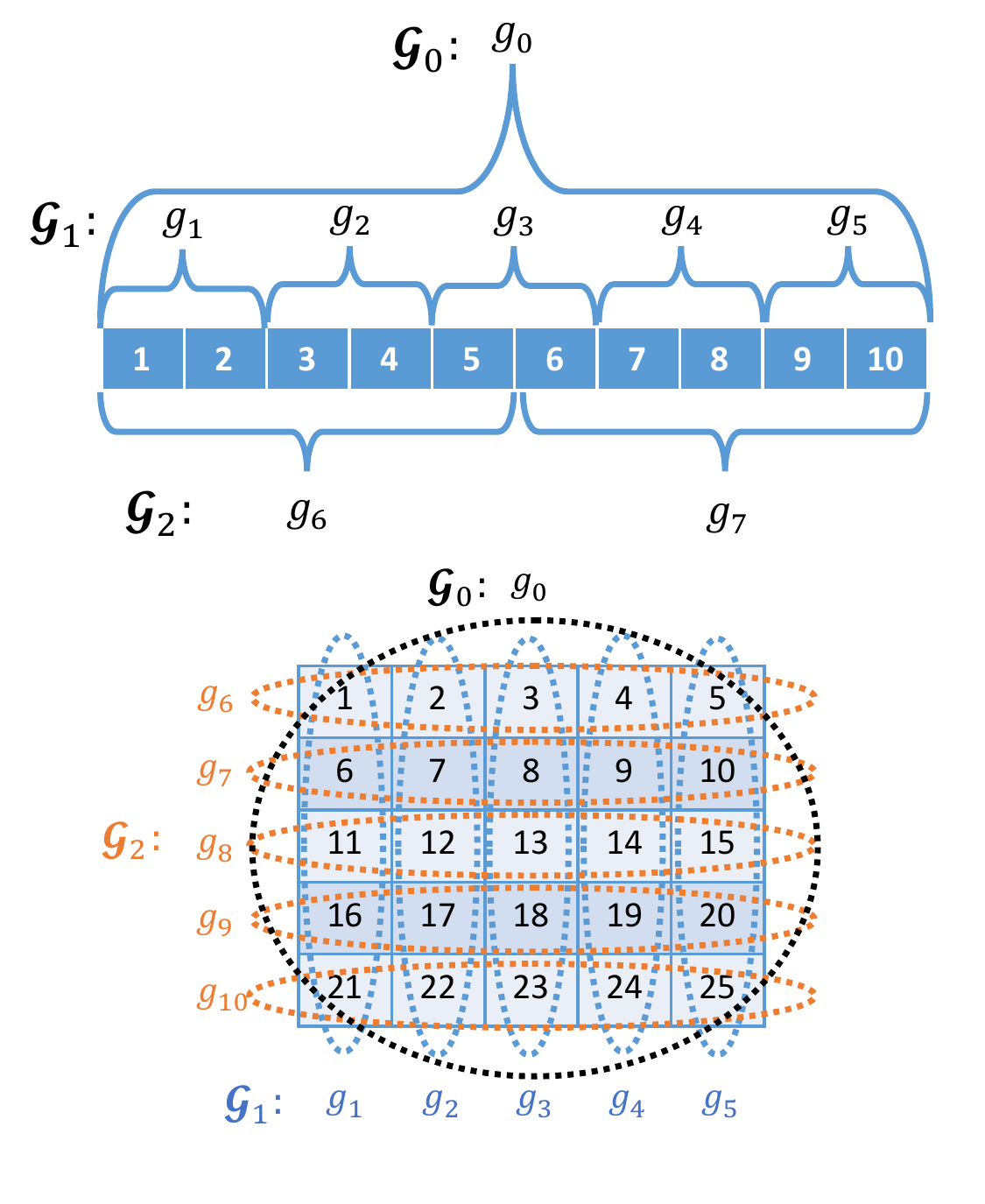}\\
\caption{Two examples of groups for TVCS model. The first model has $p=10$, $\mathcal{G}_0=\{g_0\}, \mathcal{G}_1=\{g_1,g_2,g_3,g_4,g_5\},\mathcal{G}_2=\{g_6,g_7\}$. The second model organizes elements as matrix ($p=25$). Each row and column is a group, and $\mathcal{G}_0=\{g_0\}, \mathcal{G}_1=\{g_1,g_2,g_3,g_4,g_5\},\mathcal{G}_2=\{g_6,g_7,g_8,g_9,g_{10}\}$.}\label{fig:tvcs}
\end{figure}

The TVCS model is motivated from some important applications, for example, in recommendation, task-worker assignment, and bioinformatics.
\begin{itemize}
\item {\bf Online recommendation.} Suppose we want to recommend a certain number of books (among $p$ books) to a customer - corresponding to the $\mathcal{G}_0$ based sparsity constraint. Among the selected books, we want to maintain some diversities - the recommended books by the same author should not exceed a certain number ($\mathcal{G}_1$ based sparsity constraint) and about the same topic should not exceed a certain number either ($\mathcal{G}_2$ based sparsity constraint). One can refer to the top graph in Figure~\ref{fig:tvcs}: $\mathcal{G}_1$ is grouped by authors and $\mathcal{G}_2$ is grouped by topics.
\item {\bf Task-worker assignment.} Suppose we have a bunch of tasks and workers, and we want to assign the tasks to workers. For example, in crowdsourcing, we usually assign several different workers to each task since we want to use the answers from multiple workers to improve the accuracy. On the other hand, each worker is usually assigned to multiple tasks so there is a ``many to many'' relationship in this assignment.
The goal is to pursue the optimal assignment under a certain criteria in crowdsourcing, while satisfying some restrictions. For example, the total assignments should be bounded by the total budget (corresponding to $\mathcal{G}_0$), the total cost of assignments to a single worker cannot exceed a certain threshold (corresponding to $\mathcal{G}_1$), and the total cost of assignments on a single task cannot exceed a certain threshold (corresponding to $\mathcal{G}_2$). Let $X$ be the assignment matrix, and its rows are indexed by workers and the columns are indexed by tasks. These constraints can be illustrated by the bottom graph in Figure~\ref{fig:tvcs}.
\item {\bf Identification of gene regulatory networks.} The essential goal of identifying gene regulatory network is to identify a weighted directed graph, which can be represented by a square matrix $W$ with $p=N\times N$ elements in total where $N$ is the number of vertices. A sparse network constraint is to restrict the in-degree and out-degree for each vertex, which corresponds to the sparsity in each row and column of $W$.
\end{itemize}

To solve the TVCS constrained projection~\eqref{eq:proj_obj}, we show an interesting connection between the projection and a linear programming (LP) that the vertex solution to this linear programming is an integer solution which solves the original problem.

To find an integer solution to such LP efficiently, we formulate it into a feasibility problem, and further an equivalent quadratic convex optimization. By using the rounding technique, we can avoid solving the exact solution of this LP problem. We propose an iterative algorithm to solve it and each iteration can be completed in linear time. We also show that the iterate linearly converges to the optimal point. Finally, the proposed TVCS model is validated by the synthetic experiment and two important and novel applications in identification of gene regulatory networks and task assignment problem of crowdsourcing.

\section{Related Works}
Recent years have witnessed many research works in the field of structured sparsity and group-based sparsity. \citet{yuan2006model} introduced the group LASSO, which pursues group-wise sparsity that restricts the number of groups for the selected variables. \citet{jenatton2011proximal} construct a hierarchical structure over the variables and use group LASSO with overlapped groups to solve it. Exclusive LASSO~\citep{zhou2010exclusive,kong2014exclusive} was proposed for the exclusive group sparsity which can be treated as relaxing our cardinality constraints to convex regularizations. In~\citep{kong2014exclusive}, the authors discussed the overlapping situation and tried to solve the problem using convex relaxation, which is different from our approach. Besides the aforementioned works, some proposed more general models to cover various sparsity structures. \citet{bach2012structured} extended the usage of $L_1$-norm relaxation to several different categories of structures. And recently, another generalization work~\citep{el2015totally} proposed convex envelopes for various sparsity structures. They built the framework by defining a totally unimodular penalty, and showed how to formulate different sparsity structures using the penalty. The work above concentrated on using convex relaxation to control the sparsity.

Besides using convex relaxation, there are also some works focusing on projection-based methods. When the exact projection operator was provided, \citet{baraniuk2010model} extended the traditional IHT and CoSaMP methods to general sparsity structures. In this work, the authors also introduced the projection operator for block sparsity and tree sparsity. \citet{cevher2009recovery} investigated cluster sparsity and they applied dynamic programming to solve the projection operator for their sparsity model. \citet{hegde2009compressive} introduced a ``spike trains'' signal model, which is also related to exclusive group sparsity. Its groups always have consecutive coordinates, and each group cannot contain more than one nonzero element. To solve the projection problem of their model, they showed the basic feasible solutions of the relaxed linear programming (LP) are always integer points. In our work, we also use LP to solve the projection problem. But our model defines the group structure differently and aims at different applications.

In addition, there are some works for the cases without an efficient exact projection operator~\citep{hegde2015approximation,hegde2015nearly,nguyen2014linear}. This is meaningful since the projection operator for complex structured sparsity often involves solving complicated combinatorial optimization problems. \citet{hegde2015approximation} discussed how to guarantee convergence if using approximate projection in IHT and CoSaMP for compressive sensing. They proved that the convergence needs a ``head approximation'' to project the update (gradient) before applying it. \citet{hegde2015nearly} proposed a general framework to formulate a series of models as a weighted graph, and designed an efficient approximate projection operator for the models. \citet{nguyen2014linear} applied the approximate projection-based IHT and CoSaMP to general convex functions and stochastic gradients.

\section{Preliminary: GradMP and IHT Frameworks}

This section briefly reviews two commonly used algorithm frameworks to solve the cardinality constrained optimization~\eqref{eq:main}: iterative hard thresholding (IHT)~\citep{yuan2014gradient,nguyen2014linear} and gradient matching pursuit (GradMP)~\citep{gradmp,nguyen2014linear} (the general version of CoSaMP~\citep{needell2009cosamp}) for solving cardinality constrained problem. Other methods like hard thresholding pursuit (HTP) also follows similar steps and has been shown to be effective both empirically and theoretically~\citep{yuan2016exact}. The procedures of IHT and GradMP for our model are shown in Algorithms~\ref{alg:stoiht} and \ref{alg:gradmp}, where $\supp{\cdot}$ is the support set of the argument vector.


Therefore, one can see that the efficiency of both algorithms relies on the computation of the gradient and the projection. To avoid the expensive computation of the gradient, GradMP and IHT can be extended to the stochastic versions~\citep{nguyen2014linear} by assigning $\g$ the stochastic gradient at the gradient computation step.

Both Algorithms~\ref{alg:stoiht} and ~\ref{alg:gradmp} (and their stochastic variants) guarantee some nice properties: the iterate converges to a small ball surrounding the true solution at a linear rate under certain RIP-type conditions~\citep{nguyen2014linear} and the radius of such ball converges to zero when the number of samples goes to infinity.

\begin{algorithm2e}
 \SetAlgoLined
 \KwIn{Sparsity parameter $\s$.}
 \KwResult{Problem solution $\w^t$.}
 Initialize $\w^0, t=0$;\\
 \While{stop criterion is not met}{
    $\g={\nabla} f(\w^t)$ \tcp*{Gradient computation}
    $\z^t=\w^t-\gamma \g$ \tcp*{Gradient descent}
    $\w^{t+1}=\proj_{\Omega(\mathcal{G}, \s)}(\z^t)$ \tcp*{Projection}
    $t=t+1$;\\
 }
\caption{Iterative Hard Thresholding.}
\label{alg:stoiht}
\end{algorithm2e}

\begin{algorithm2e}
\label{alg:gradmp}
 \SetAlgoLined
 \KwIn{Sparsity parameter $\s$.}
 \KwResult{Problem solution $\w^t$.}
 Initialize $\w^0, t=0$;\\
 \While{stop criterion is not met}{
    $\g=\nabla f(\w^t)$\tcp*{Gradient computation}
    $\Gamma=\supp{\proj_{\Omega(\mathcal{G}, 2\s)}(\g)}$; \\
    $\hat{\Gamma}=\Gamma \cup \supp{\w^t}$  \tcp*{Subspace selection}
    $\z^t=\argmin_{\supp{\z}=\hat{\Gamma}} f(\z)$ \tcp*{Subspace optimization}
    $\w^{t+1}=\proj_{\Omega(\mathcal{G}, \s)}(\z^t)$ \tcp*{Projection}
    $t=t+1$;\\
 }
\caption{Gradient Matching Pursuit.}
\label{alg:gradmp}
\end{algorithm2e}

A common component in Algorithms~\ref{alg:stoiht} and \ref{alg:gradmp} is the projection operator. If all the groups except $[p]$ in $\mathcal{G}$ do not overlap each other, the projection problem can be easily solved by sequential projections~\citep{yangbenefits}. But for those cases involving overlapped groups, it is generally challenging to solve them efficiently.

\section{Projection Operator}
This section introduces how to solve the essential projection step. Note that the projection onto a nonconvex  set is NP-hard in general. By utilizing the special structure of TVCS, we show that the  projection can be solved efficiently.
Due to the page limitation, all proofs are provided in the supplementary material.

\subsection{LP Relaxation}
Firstly, we can cast the projection problem~\eqref{eq:proj_obj} to an equivalent integer linear programming problem (ILP) according to Lemma~\ref{thm:proj_ip}.

\begin{lemma}\label{thm:proj_ip}
The projection problem~\eqref{eq:proj_obj} is equivalent to the following integer linear programming problem (ILP):
\begin{align}
\max_\x \quad & \langle \v^2, \x \rangle \label{eq:proj_ilp}\\
\text{subject to}\quad & A\x \leq  \s \notag\\
& {\x\in \{0,1\}^p} \notag
\end{align}
where $\v^2$ is applying element-wise square operation on vector $\v$. $A$ is a $|\mathcal{G}| \times p$ matrix which is defined as:
\begin{equation}
A=
\begin{bmatrix}
\1^\top \\
C
\end{bmatrix}
\label{eq:A_def}
\end{equation}
where $C\in \{0,1\}^{|\mathcal{G}_1 \cup \mathcal{G}_2| \times p}$, whose rows represent the indicator vector of each group $g\in \mathcal{G}_1$ and $\mathcal{G}_2$.

Each row in $A$ corresponds to one group $g$ from $\mathcal{G}$.
For example, $C_{ij}=1$ if the $j$-th coordinate is in the $i$-th group, otherwise $C_{ij}=0$.
The first row $\1^\top$  corresponds to the overall sparsity i.e. $\mathcal{G}_0$.

\end{lemma}

%

It is NP-hard to solve an ILP in general. One common way to handle such ILP is making a linear programming (LP) relaxation. In our case, we can use a box constraint $\x\in[0,1]^p$ to replace the integer constraint $\x\in\{0,1\}^p$:
\begin{align}
\max_\x \quad & \langle \v^2, \x \rangle \label{eq:proj_lp}\\
\text{subject to}\quad & A\x \leq  \s \notag\\
& {\x\in [0,1]^p} \notag
\end{align}
However, there is no guarantee that a general ILP can be solved via its LP relaxation, because the solution of the relaxed LP is not always integer. Although one can make a rounding to the LP solution and acquire a integer solution, such solution is not guaranteed to be optimal (or even feasible) to the original ILP.

Fortunately, due to the special structure of our TVCS model, we find that its relaxed LP has some nice properties which make it possible to get the optimal solution of the ILP efficiently. The following theorem reveals the relationship between the ILP problem and the relaxed LP problem.

\begin{theorem}\label{thm:iplp}
Given $\mathcal{G}$ satisfying $TVCS$, all the vertices of the feasible set to \eqref{eq:proj_lp} are integer points. Furthermore, there is an optimal solution on the vertex that solves the ILP~\eqref{eq:proj_ilp}.
\end{theorem}
This theorem suggests that finding a vertex solution of the relaxed LP can solve the original projection problem onto a TVCS $\mathcal{G}$. The proof basically shows that matrix $A$ (for TVCS) is a totally unimodular matrix~\citep{papadimitriou1982combinatorial}. We provide the detailed proof in the supplementary material.

\subsection{Linearly Convergent Algorithm for Projection Operator onto TVCS}
To find a solution on the vertex, one can use the Simplex method. Although Simplex method guarantees to find an optimal solution on the vertex and could be very efficient in practice, it does not have a deterministic complexity bound. In the IHT and GradMP algorithms, projection operator is only a sub-procedure in one iteration. Hence, we are usually supposed to solve lots of instances of problem~\eqref{eq:proj_ilp}. Simplex might be efficient practically, but its worst case may lead to exponential time complexity~\citep{papadimitriou1982combinatorial}. In this section, the integer solution to the linear programming can be found within the complexity proportional to the number of variables and constraints.

\paragraph{Equivalent Feasibility Problem Formulation.} The dual of LP problem~\eqref{eq:proj_lp} can be written as:
\begin{align}
\min_\y \quad & \langle [\s^\top\ \1^\top]^\top, \y \rangle \label{eq:proj_lpdual}\\
\text{subject to}\quad & \begin{bmatrix} A^\top \ I \end{bmatrix} \y \geq \v^2, \y \geq \0 \notag
\end{align}

Since the duality gap of LP is zero, combining the primal LP~\eqref{eq:proj_lp} and dual LP~\eqref{eq:proj_lpdual}, we can formulate an equivalent problem, i.e. the feasibility problem over the following constraints:
\begin{align*}
\text{find}\quad & \x, \y \\
\text{subject to}\quad & \langle [\s^\top\ \1^\top]^\top, \y \rangle = \langle \v^2, \x \rangle\\
& \begin{bmatrix} A^\top \ I \end{bmatrix} \y \geq \v^2 \\
& \begin{bmatrix} A \\ I \end{bmatrix}\x \leq \begin{bmatrix} \s \\ \1 \end{bmatrix}\\
& \y \geq \0, \x \geq \0
\end{align*}

\paragraph{Iterative Algorithm.}
The feasibility problem with linear constraints above is equivalent to the following optimization problem:
\begin{align}
\min_{\x, \y}\quad & {1\over 2} (\langle [\s^\top\ \1^\top]^\top, \y \rangle - \langle \v^2, \x \rangle)^2 \notag\\
& + {1\over 2} \| [\v^2 - [A^\top \  I]\y]_+ \|^2 + {1\over 2} \| [A\x-\s]_+ \|^2 \label{eq:lpgd}\\
\text{subject to}\quad & \0 \leq \x \leq \1, \y \geq 0 \notag
\end{align}
where $[\z]_+$ is the element-wise hinge operator, i.e. it transforms each element $\z_i$ to $\max(\z_i, 0)$.

This is a convex optimization problem with {\bbc a} quadratic objective and box constraints. We adopt the projected gradient descent to solve this problem, and show it converges linearly.

\begin{theorem}\label{thm:iteralgo}
For the optimization problem with the form
\begin{align*}
\min_\z \quad & f(\z) := \|[A\z-a]_+ \|^2 + \|B\z-b\|^2\\
\text{subject to} \quad & \z\in \Omega
\end{align*}
where $\Omega = \{\z \mid C\z \leq c\}$, the projected gradient descent algorithm $\z^{t+1} \leftarrow P_\Omega(\z^t - \gamma\nabla f(\z^t))$ has a linear convergence rate with some $\alpha<1$ (depending on $A$ and $B$):
$$\|\z^{t+1}-P_{\z^*}(\z^{t+1})\| \leq \alpha \|\z^{t}-P_{\z^*}(\z^{t})\|,$$
where $P_{\z^*}(\cdot)$ is the projection onto the optimal solution set.
\end{theorem}
Notice that the objective function $f$ in Theorem~\ref{thm:iteralgo} is not necessarily strongly convex, which means the well recognized linear convergence conclusion from the strong convexity is not applicable here.

Theorem~\ref{thm:iteralgo} mainly applies Hoffman's Theorem \citep{hoffman2003approximate} to show that $f$ is an optimal strongly convex function \citep{liu2015asynchronous}. This leads to a linear convergence rate.

The convergence rate $\alpha=\nicefrac{1}{(1+{\lambda \over L})}$, where $\lambda$ is the Hoffman constant~\citep{hoffman2003approximate} that depends on $A,B$ and is always positive. $L$ is the Lipschitz continuous gradient constant. More details are included in the supplementary materials.

To show the complexity of this algorithm, we firstly count how many iterations we need. Since we know that we can just make a rounding\footnote{Acute readers may notice that the convergent point may be on the face of the polytope in some cases instead of vertex. However, we can add a small random perturbation to ensure the optimal point to be vertices with probability 1.} to the result $\x^t$ when we attain $\|\x^t-\tilde{\x}^*\|_\infty < 0.5$. Let $\z:=[\x^\top\ \y^\top]^\top$ represent all the variables in~\eqref{eq:lpgd}. Because $\|\z^t-{\z}^*\| \geq \|\z^t-{\z}^*\|_\infty \geq \|\x^t-{\x}^*\|_\infty$, we can do the rounding safely when $\|\z^t-{\z}^*\| < 0.5$, where ${\z}^*, {\x}^*$ are the optimal points of this problem. According to Theorem~\ref{thm:iteralgo}, we have the linear convergence rate $\alpha<1$, so the number of iterations we need is
\[
t > \log_{\alpha} {1\over 2\|\z^0-{\z}^*\|}
\]
{\bbc Therefore,} we {\bbc claim} that we can {\bbc obtain} the solution ${\x}^*$ by rounding {\bbc after $\log_{\alpha} {1\over 2\|\z^0-{\z}^*\|}$ iterations}.

Secondly, we show that the computation {\bbc complexity} in each iteration is linear with dimensionality $p$ and the {\bbc amount} of groups $|\mathcal{G}|$. Since each column of $A$ contains at most 3 nonzero elements, {\bbc the complexity of the matrix multiplications in computing the gradient of \eqref{eq:lpgd} is $O(p+|\mathcal{G}|)$}.
Together with other computation, the complexity for each iteration is $O(p+|\mathcal{G}|)$.
\section{Empirical Study}

{\rrc This section will validate the proposed method on both synthetic data and two practical applications: crowdsourcing and identification of gene regulatory networks.}
\subsection{Linear Regression and Classification on Synthetic Data}
\label{sec:linearexp}
In this section, we validate the proposed {\rrc method with linear regression objective and squared hinge objective (classification) on} synthetic data. {\rrc Let $\w \in \R^{\sqrt{p} \times \sqrt{p}}$ be a matrix, $\mathcal{G}_1$ and $\mathcal{G}_2$ are defined as groups with all rows and all columns respectively. The linear regression loss is defined as $\sum_{i=1}^n (\langle X_i, \w \rangle - y_i)^2$ and the squared hinge loss is defined as $\sum_{i=1}^n \max(0, 1-y_i\langle X_i, \w \rangle)^2$, where $n$ is the total number of training samples. $X_i$ and $y_i$ are the features and label of the $i$-th sample respectively.}

In the linear regression experiment, the true model $\bar{\w} \in \R^{\sqrt{p}\times \sqrt{p}}$ is generated from the following procedure: generate a random vector and apply the projection operator to get a support set which satisfy our sparsity constraints; the elements of positions in support set are drawn from standard normal distribution. $p$ is fixed as $400$ and $n$ is gradually increased. The group sparsity upper bounds $\s^g$ for $g\in\mathcal{G}_1$ and $g\in\mathcal{G}_2$ are uniformly generated from the integers in the range$[1,\sqrt{p}]$. The overall sparsity upper bound is set by $0.8\times\min(\sum_{g\in \mathcal{G}_1}\s^g, \sum_{g\in \mathcal{G}_2}\s^g)$. Each $X_i$'s is an $\sqrt{p} \times \sqrt{p}$ i.i.d. Gaussian random matrix. $y_i$ is generated from $y_i = \langle X_i, \bar{\w}\rangle + e_i$, where $e_i$ is the i.i.d. Gaussian random noise drawn from $\mathcal{N}(0,0.01^2)$. We compare the proposed methods to bilevel exclusive sparsity with non-overlapped groups (row-wise or column-wise)~\citep{yangbenefits}, overall sparsity~\citep{needell2009cosamp}, and exclusive LASSO~\citep{kong2014exclusive}. For fairness we project the final result of all the compared methods to satisfy all constraints.  All the experiments are repeated 30 times and we use the averaged result. We use selection recall and successful recovery rate to evaluate the performance. Selection recall is defined as $\nicefrac{|\supp{\w^*}\cap\supp{\bar{\w}}|}{\|\bar{\w}\|_0}$, where $\w^*$ is the optimization result. Successful recovery rate is the ratio of the successful feature selection i.e. $\supp{\w^*}=\supp{\bar{\w}}$ to the total number of repeated experiments. In Figure~\ref{fig:ls} we can observe that our model with all sparsity constraints always have the best performance. While the performance of exclusive LASSO and our method is comparable when the number of samples are very limited, our method outperforms exclusive LASSO when the number of samples increases.

For classification experiments, we use the same settings of sparsity with linear regression. Here we set $p=400$, and change $n$ from $200$ to $800$. The true model $\bar{\w}$ and feature matrices are generated by the same way as the linear regression experiment. The class label $y_i \in \{-1,1\}$ is got by $y_i=\text{sign}\langle X_i, \bar{\w}\rangle$. Besides the selection recall, we also compare the classification error. In Figure~\ref{fig:lc}, we can see that the superiority of our method is even more significant in the classification experiment. Although the overall sparsity has the lowest selection recall, it still has a similar classification error with the methods that consider row or column groups.

\begin{figure}[htbp]
	\begin{minipage}[b]{0.49\linewidth}
		\centering
		\centerline{\includegraphics[width=1\linewidth]{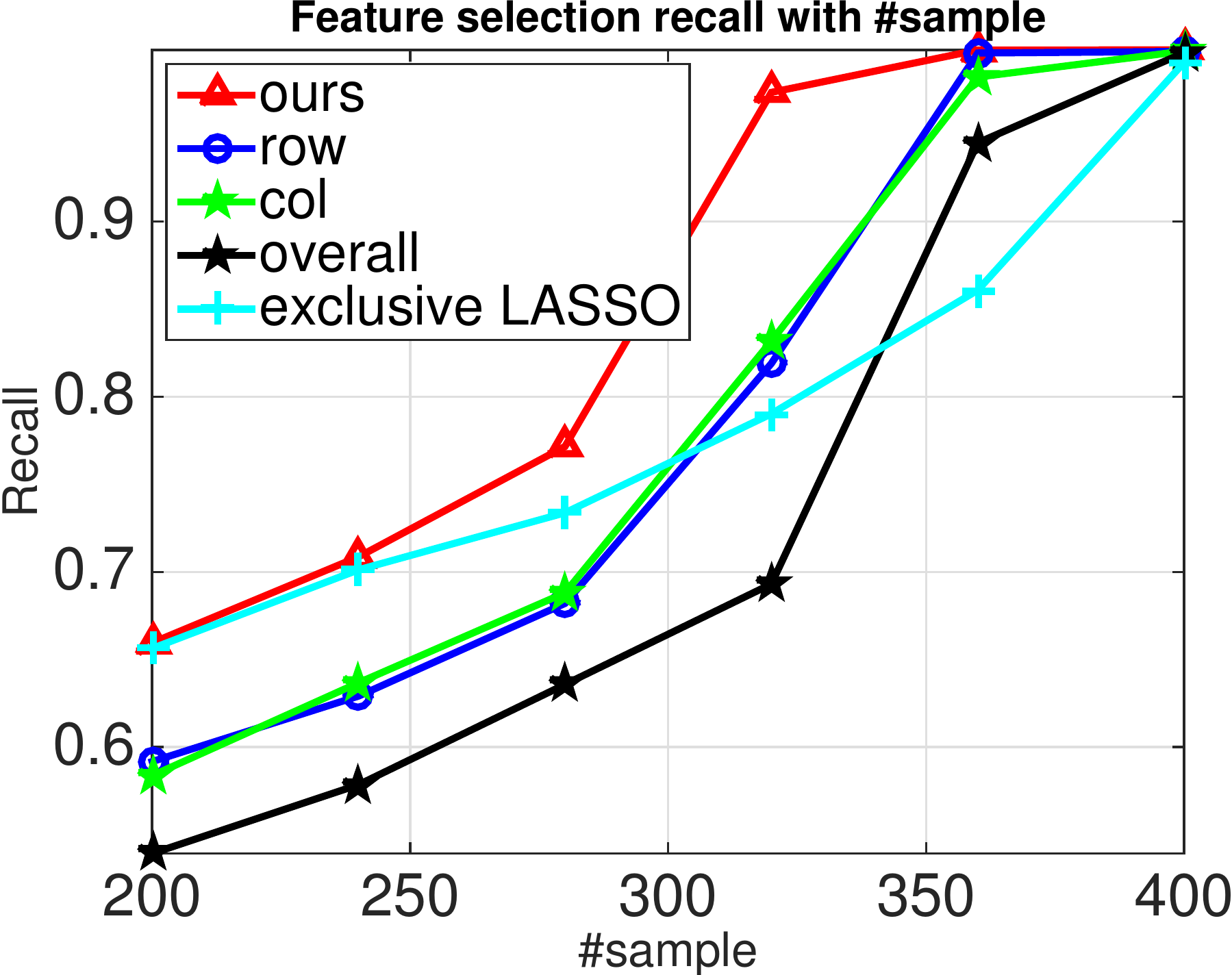}}
		\centerline{\scriptsize{(a) Selection recall.}}\medskip
	\end{minipage}
	\begin{minipage}[b]{0.49\linewidth}
		\centering
		\centerline{\includegraphics[width=1\linewidth]{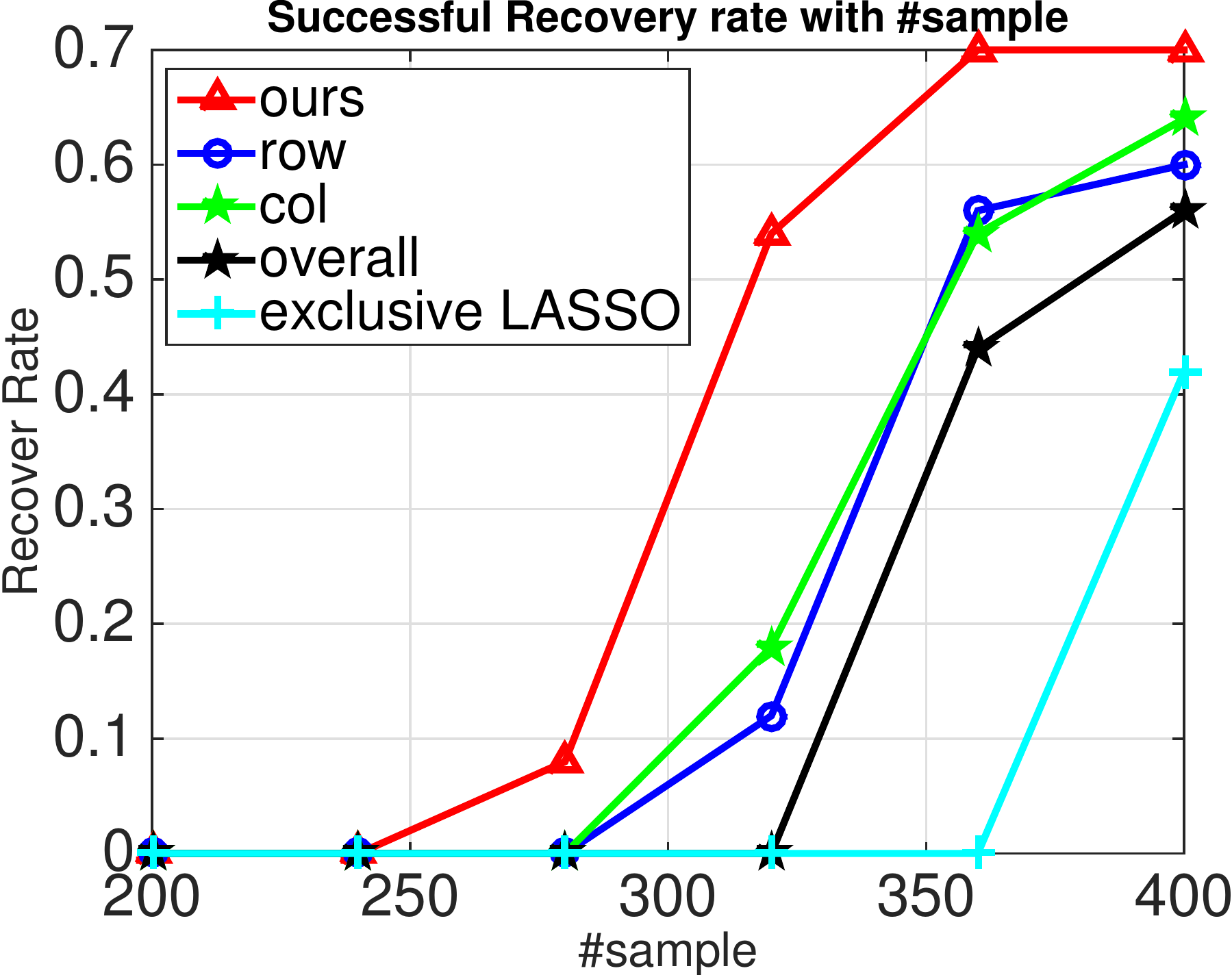}}
		\centerline{\scriptsize{(b) Successful recovery rate.}}\medskip
	\end{minipage}
	\caption{Selection recall and successful recovery rate for least square loss.}
	\label{fig:ls}
\end{figure}

\begin{figure}[htbp]
	\begin{minipage}[b]{0.49\linewidth}
		\centering
		\centerline{\includegraphics[width=1\linewidth]{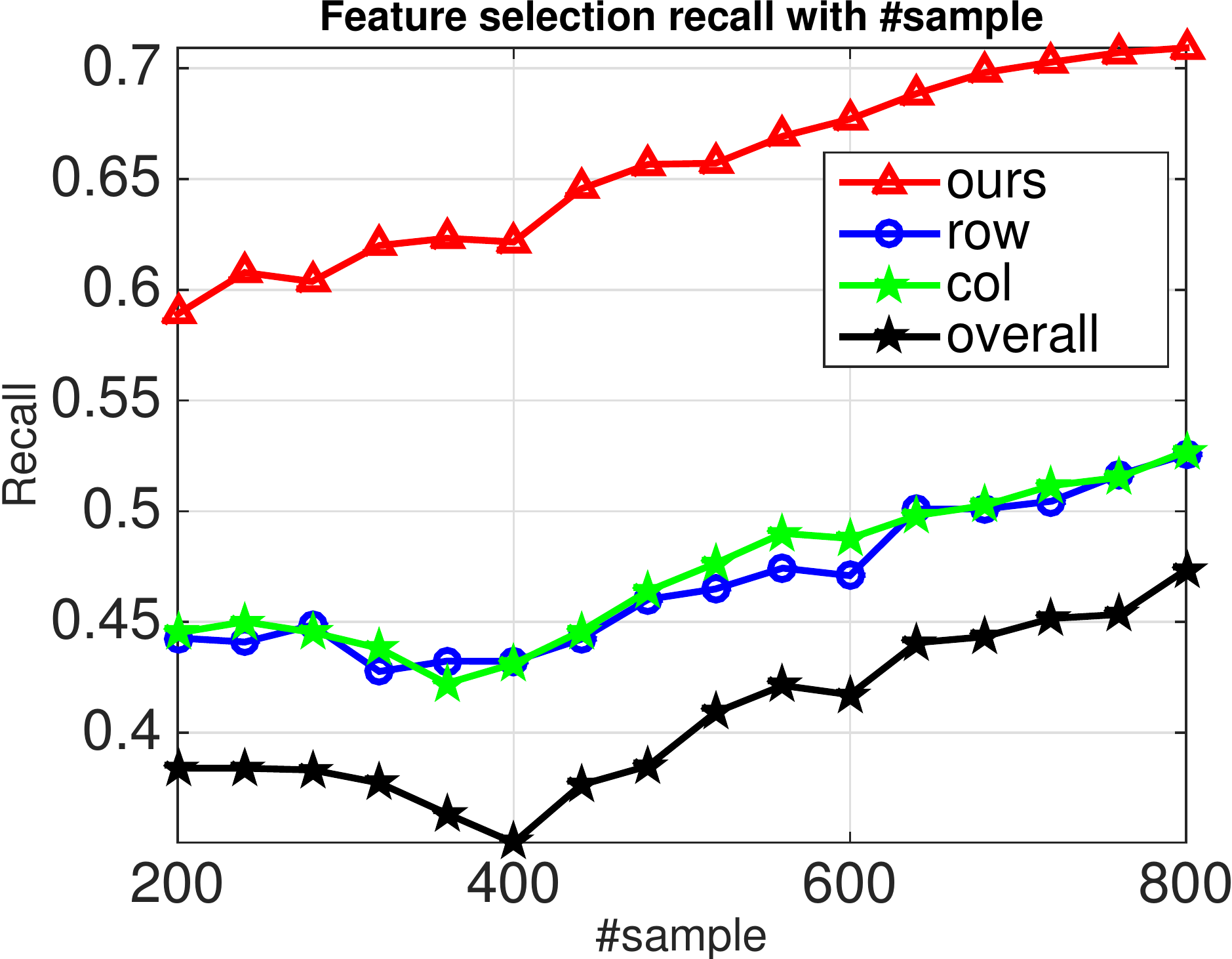}}
		\centerline{\scriptsize{(a) Selection recall.}}\medskip
	\end{minipage}
	\begin{minipage}[b]{0.49\linewidth}
		\centering
		\centerline{\includegraphics[width=1\linewidth]{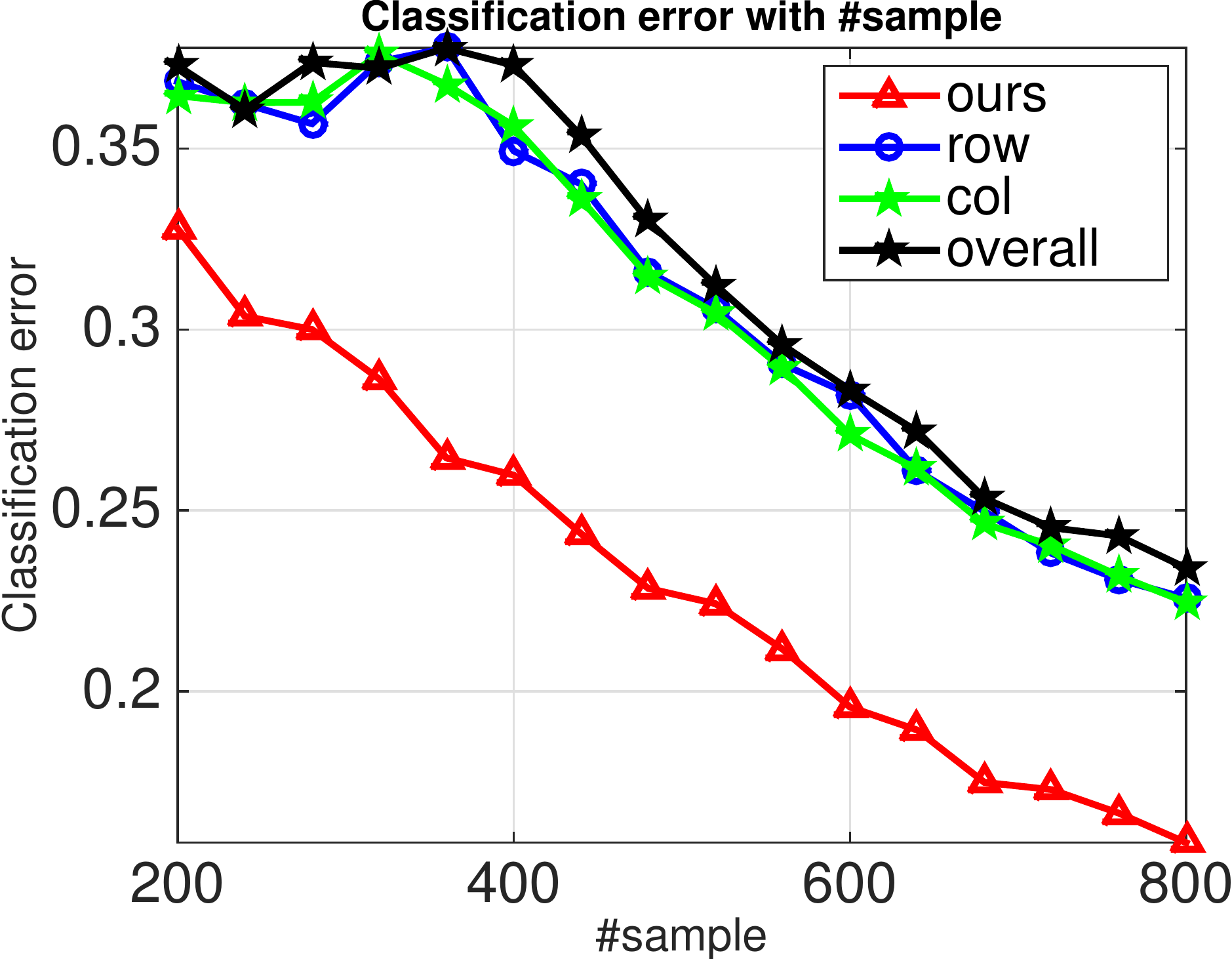}}
		\centerline{\scriptsize{(b) Classification error.}}\medskip
	\end{minipage}
	\caption{Selection recall and classification error for squared hinge loss.}
	\label{fig:lc}
\end{figure}

\subsection{Application in Crowdsourcing}
\label{sec:crowd}
This section applies the proposed method to the worker-task assignment problem in crowdsourcing. Take the image labeling task as an example. Given $n$ workers and $m$ images, each image can be assigned to multiple workers and each worker can label multiple images. The predicted label for each image is decided by all the labels provided by the assigned workers and the quality of each worker on the image. The goal is to maximize the expected prediction accuracy based on the assignment. Let $X\in \{0,1\}^{n\times m}$ be the assignment matrix, i.e. $X_{ij}=1$ if assign the $i$-th worker to $j$-th task, otherwise $X_{ij}=0$. $Q\in [0,1]^{n\times m}$ is the corresponding quality matrix, which is usually estimated from the golden standard test \citep{ho2013adaptive}. The whole formulation is defined to maximize the average of the expected prediction accuracy over $m$ tasks over a TVCS constraint:
\begin{align}
\max_{X}\quad &{1\over m}\sum_{j=1}^{m}\mathcal{E}_{\text{acc}}(Q_{\cdot,j}, X_{\cdot,j}) \label{eq:crowd}
\\
\text{subject to }\quad  &\sum_{i=1}^n X_{ij} \leq \s^{\text{worker}}, \forall j; \notag\\
& \sum_{j=1}^m X_{ij}\leq \s^{\text{task}}, \forall i; \notag\\
&\sum_{i,j}X_{ij} \leq \s^{\text{total}};\ X \in \{0,1\}^{n\times m}\notag
\end{align}
where $\mathcal{E}_{\text{acc}}(\cdot, \cdot)$ is the expected prediction accuracy, $\s^{\text{worker}}$ is the ``worker sparsity'', i.e. the largest number of assigned workers for each task, and $\s^\text{task}$ is the ``task sparsity'', i.e. each worker can be assigned with at most $\s^{\text{task}}$ tasks, and $\s^{\text{total}}$ is the total sparsity to control the budget, i.e., the maximal number of assignment. In image labeling task, we assume that each image can only have two possible classes and the percentage of images in each class is one half. We use the Bayesian rule to infer the predicted labels given the workers' answer. Here we consider the binary classification task. Let $\y_j \in \{1,0\}$ be the true label of the $j$-th task and $\hat{\y}_j$ be the prediction given labels by selected workers, i.e.,
$$
\hat{\y}_j=
\begin{cases}
0, &\text{ if } \P(\y_j=1 | \hat{Y}_{\Omega_j,j}) < \P(\y_j=0 | \hat{Y}_{\Omega_j,j});\\
1, &\text{ otherwise }
\end{cases}
$$
where $\hat{Y}_{ij}$ is the $i$-th worker's predication {\rc on} $j$-th task. Set $\Omega_j$ contains the indices of the selected workers for $j$-th task, i.e. $X_{ij}=1, \forall i \in \Omega_j$, and $X_{i'j}=0, \forall i' \notin \Omega_j$

Then $\mathcal{E}_{\text{acc}}(Q_{\cdot,j}, X_{\cdot,j})$ can be defined in the following:
\begin{equation*}
\mathcal{E}_{\text{acc}}(Q_{\cdot,j}, X_{\cdot,j})=\P(\hat{\y}_j=1, \y_j=1 )+ \P(\hat{\y}_j=0 ,\y_j=0 )
\end{equation*}
By this way, the expected accuracy will not be continuous, so we use smooth function to approximate the expected accuracy and adopt the stochastic gradient with the proposed projection operator to optimize it. Due to the space limitation, the detailed derivation of the objective formulation can be found in the supplemental material.

We conduct experiment for crowdsourcing task assignment on synthetic data. Specifically, we generate the quality matrix $Q$ from uniformly random distribution with interval $[0.5, 0.9]$. The prior probability $\P(\y_j=1)$ and $\P(\y_j=0)$ are set as $0.5$ for all the tasks.

To avoid evaluating the expectation term, we apply the stochastic iterative hard thresholding framework \citep{nguyen2014linear}. Each iteration we get $\hat{Y} | \y_j=1$ and $\hat{Y} | \y_j=0$ by sampling based on $Q$ i.e. $\P(\hat{Y}_{ij}=1|\y_j=1)=Q_{ij}$, $\P(\hat{Y}_{ij}=0|\y_j=0)=Q_{ij}$. Then we can get a stochastic gradient based on the sampled $\hat{Y}$.

\begin{figure}[htbp!]
	\begin{minipage}[b]{0.49\linewidth}
		\centering
		\centerline{\includegraphics[width=1\linewidth]{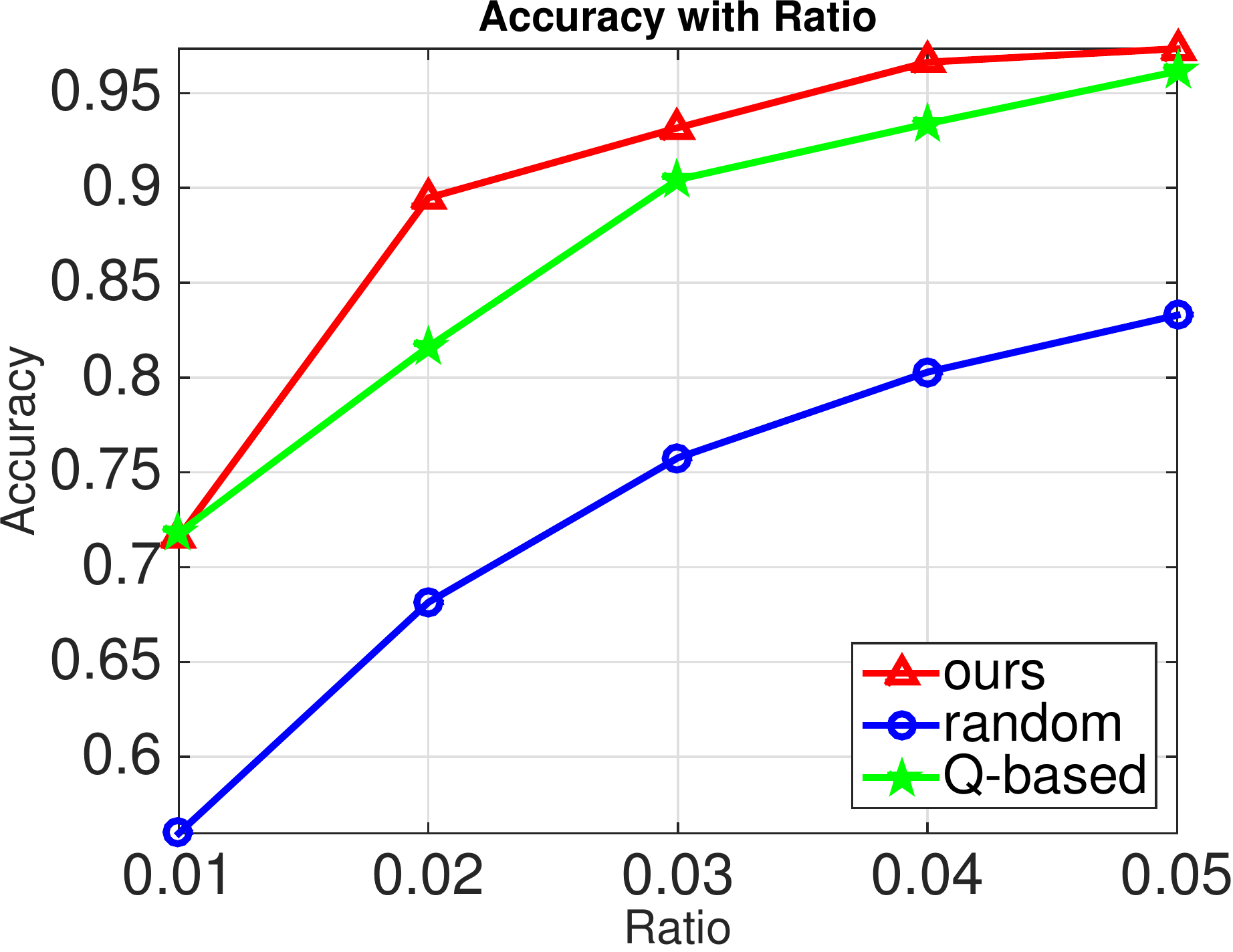}}
		\centerline{\scriptsize{(a) $n=100$, $m=1,000$. }}\medskip
	\end{minipage}
	\begin{minipage}[b]{0.49\linewidth}
		\centering
		\centerline{\includegraphics[width=1\linewidth]{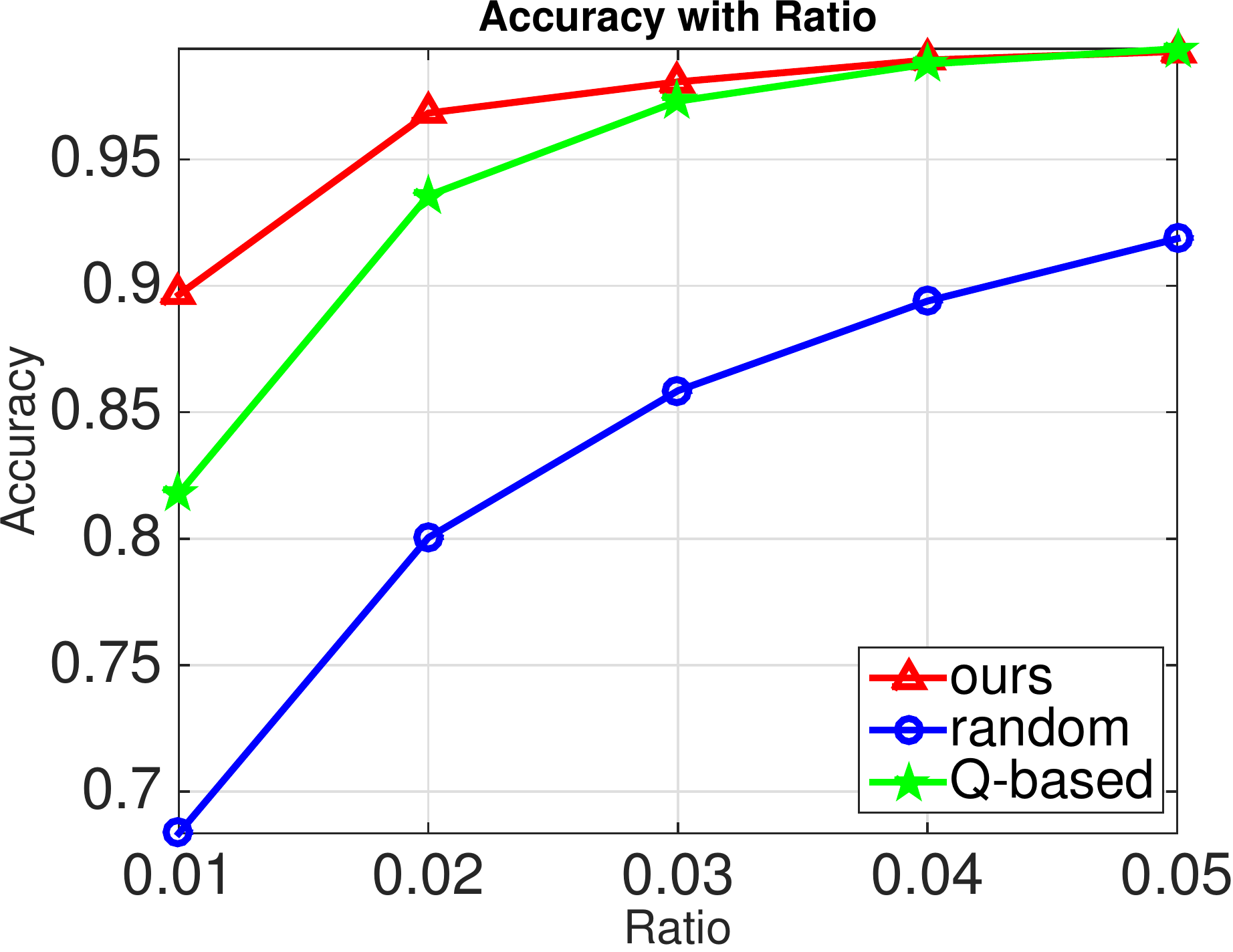}}
		\centerline{\scriptsize{(b) $n=200$, $m=10,000$.}}\medskip
	\end{minipage}
	\caption{Expected accuracy of crowdsourced classification.}
	\label{fig:crowds}
\end{figure}

Besides the proposed formulation~\eqref{eq:crowd}, we evaluate the random assignment algorithm and the Q-based linear {\rrc programming} \citep{ho2013adaptive}. The random assignment algorithm widely used in practice is the most straightforward approach: given the total assignment budget $\s^{\text{total}}$ and the restrictions ($\s^{\text{worker}}$ and $\s^{\text{task}}$) for workers and tasks, randomly assign tasks to the workers. The Q-based linear programming uses the linear combination of $Q_{ij}$ over $i$ to evaluate the overall accuracy on task $j$ for simpler formulation. In addition, it does not consider the restriction on tasks, thus it may assign lots of workers to a difficult task\footnote{A ``difficult'' task means that all workers' qualities are low on this task.}. To make a fair comparison, the task restriction is added into this method. To get the assignment result which satisfies the task and worker restriction, we use our projection operator in the other methods too.

We evaluate the experiments on different value of $\s^{\text{task}}, \s^{\text{worker}}$ by setting them as different ratios of the total number of tasks and workers. The overall sparsity is set by the same way as in Section~\ref{sec:linearexp}. To measure the performance, we compare the sampled expected accuracy. The samples (i.e., $\hat{Y}$) are independent to the samples used in training. Figure~\ref{fig:crowds} shows the comparison of the expected accuracy of the three approaches. We can observe that the accuracy increases with larger ratio (i.e. more assignments). The random assignment strategy needs more assignments to get the same accuracy compared with the other two methods.

\subsection{Application in Identification of Gene Regulatory Networks}
\label{sec:gene}
In this section, we apply the projection operator to the identification of gene regulatory networks (GRN).
\paragraph{Background.}
Gene regulatory network represents the relations between different genes, which plays important roles in biological processes and activities by controlling the expression level of RNAs. There is a well-known biological competition named DREAM challenge about {\tt identifying} GRN. Based on the time series gene expression data which are RNAs' level along time sequence, contestants are required to recover the whole gene network of given size. One popular way to infer GRN is to utilize the sparse property of GRN: e.g., one gene in the network is only related to a small number of genes and we already know that there exists no relationship between some genes. Therefore, the amount of edges connecting to one vertex is far less than the dimension of the graph. It is a practical case  of row-wise and column-wise sparsity for matrix. We could apply the projection operator to constrain the number of edges related to each vertex to identify the whole network. Recently,
the dynamic Bayesian network (DBN) \citep{zou2005new} is supposed to be an effective model to recover GRNs. The RNAs' level of all genes in GRN at time $t$ is stored in gene expression vector $\textbf{x}_t \in \mathbb{R}^{N}$, where each entry corresponds to one gene respectively and $N$ is the number of genes in GRN. Hence, We define the total amount of time points in the experiment as $T$. Gene activity model {\rrc is usually assumed to be}
\[
\x_{t+1} = P\x_t + \e_t, \quad t = 1 \ldots T-1,
\]
where $P\in \mathbb{R}^{N\times N}$ is the covariance matrix of GRN and $\e_t\in \mathbb{R}^{N}$ is Gaussian white noise. 
Then the difference of RNA levels between time points {\rrc $t+1$ and $t$, i.e. $\y_{t+1,t} \in \mathbb{R}^{N}$ is defined as follows:}
\[
\y_{t+1,t} := \x_{t+1}-\x_t = \bar{\W} \x_t + \e_t, \quad t = 1 \ldots T-1,
\]
where $\bar{\W} = P - I$ is the true sparse $N $-by-$ N$ matrix. Therefore, the GRN is only considered between different genes and we eliminate edges whose start and end vertex are the same. We define that $Y:= [\y_{2,1},\ldots,\y_{T,T-1}]\in \R^{N\times (T-1)}$ and $X:= [\x_{1},\ldots,\x_{T-1}] \in \R^{N\times (T-1)}$. The objective function is 
\[
f(\W) = {1 \over 2}\|Y-\bar{\W} X\|_{\mathcal{F}}^{2} =  {1 \over 2}\sum_{t =1}^{T-1}\|(\x_{t+1}-\x_t) - \bar{\W}\x_t\|^2.
\]

\paragraph{Time-course Gene Expression Data.}
To evaluate our method, we employ GeneNetWeaver \citep{marbach2009generating, schaffter2011genenetweaver}, the official DREAM Challenge tool for time-series expression data generation. With typical gene network structure and ordinary differential equation (ODE) models, GeneNetWeaver will produce the time-course gene expression data at pre-specified time points. In the simulation studies, we control the size of gene network to be $N= 30$ vertexes and the gene expression data are generated under 10\% Gaussian white noise.

The network is shown in Figure~\ref{fig:GeneNet30}. In this Figure, it is clear that one gene only has a few connections to other genes. Therefore, the GRN is sparse and we are able to restrict the in-degree and out-degree of every vertex by representing the network as a matrix and controlling the sparsity within each row and column.

\begin{figure}[htbp]
\centering
\includegraphics[width = 0.45\linewidth]{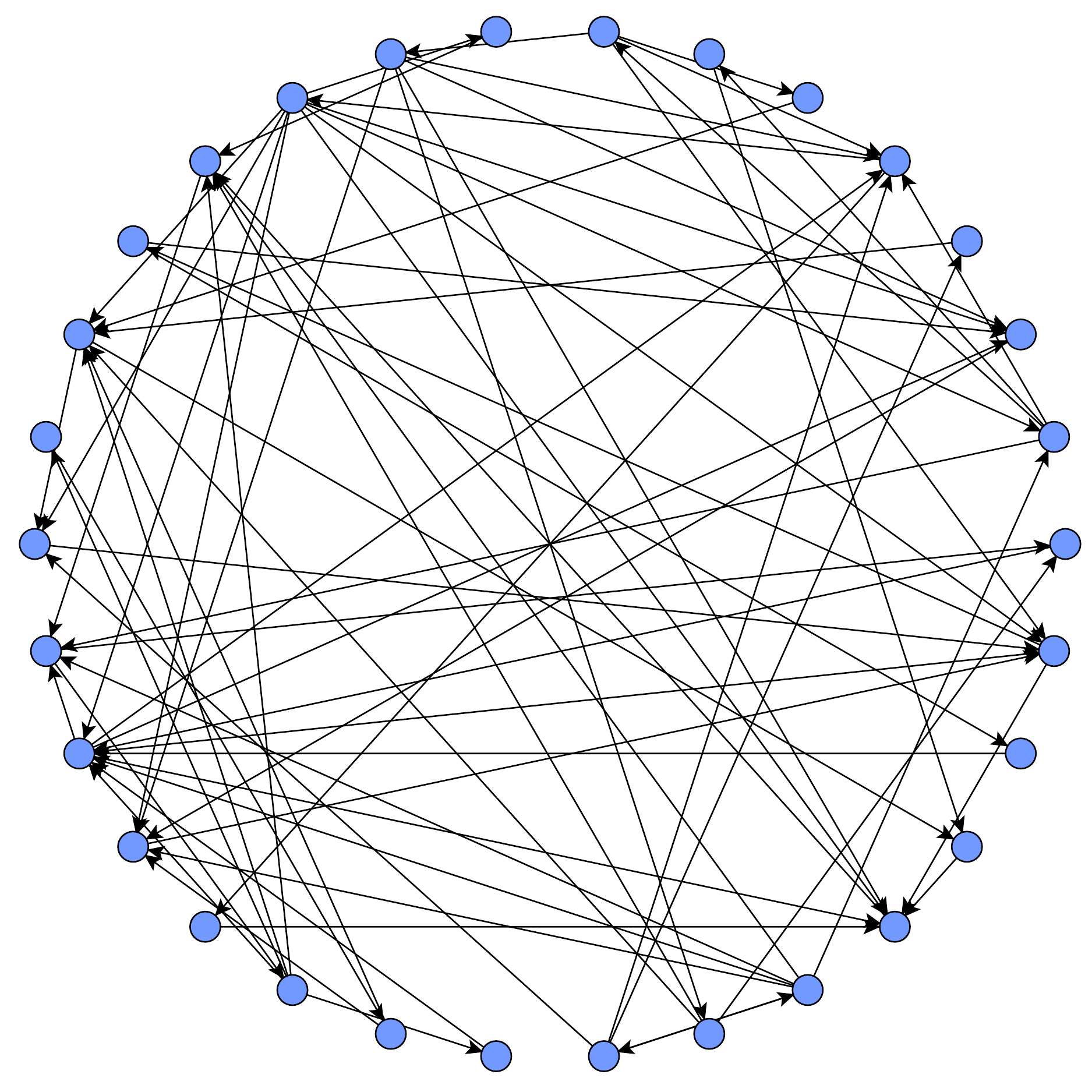}\\
\caption{This gene regulatory network contains 30 vertexes which are represented by blue circles. One edge starts at gene 1 and ends at gene 2 if gene 1 has influence on gene 2.}\label{fig:GeneNet30}
\end{figure}

\paragraph{Performance evaluation.}
Six commonly-used criteria are considered to measure the performance, i.e., sensitivity (\text{SN}), specificity (\text{SP}), accuracy (\text{ACC}), \text{F-measure}, Matthews correlation coefficient (\text{MCC}), and the Area Under ROC Curve (\text{AUC}):
\begin{align*}
	&\text{SN} = \frac{\text{TP}}{\text{TP}+\text{FN}},\quad
	\text{SP} = \frac{\text{TN}}{\text{TN} + \text{FP}},\\
	&\text{ACC} = \frac{\text{TP}+\text{TN}}{\text{TP}+\text{FP}+\text{TN}+\text{FN}},
	\\ 
	&\text{F-measure} = \frac{2\times \text{SN} \times \text{SP}}{\text{SN}+\text{SP}},\\
	&\text{MCC} =\frac{\text{TP}\times \text{TN} - \text{FP}\times \text{FN}}{\sqrt{(\text{TP}+\text{FP})(\text{TP}+\text{FN})(\text{TN}+\text{FP})(\text{TN}+\text{FN})}},
\end{align*}
where TP and TN denote the true positive and true negative, and FP and FN denote the false positive and false negative, respectively. With these criteria, we compare the performance of our method with six representative algorithms, including PCC, ARACNE \citep{margolin2006aracne}, CLR \citep{faith2007large}, MINET \citep{meyer2008minet}, GENIE3 \citep{irrthum2010inferring}, TIGRESS \citep{haury2012tigress}. The results are summarized in Table \ref{tab:performance_proj_method}. Our method outperforms other six state-of-art methods:  the AUC of our method achieve 0.7 higher which is far more than other methods; 5 out of 6 different measure show that our method has significant advantage compared to other algorithms.
\begin{table*}[tbp]
\centering
\begin{scriptsize}
\begin{tabular}{cccccccc}
\hline
&\multicolumn{1}{c}{SN}         & \multicolumn{1}{c}{SP}                & \multicolumn{1}{c}{ACC}               & \multicolumn{1}{c}{F-measure}         & \multicolumn{1}{c}{MCC}               & \multicolumn{1}{c}{AUC}                                  \\ \hline
\textbf{Our Method} & \textbf{0.6875$\pm$0.0295} & 0.7397$\pm$0.0319          & \textbf{0.7119$\pm$0.0305} & \textbf{0.7126$\pm$0.0306} & \textbf{0.4264$\pm$0.0611} & \textbf{0.7136$\pm$0.0306} \\
GENIE3              & 0.5611$\pm$0.0277          & 0.4984$\pm$0.0547          & 0.5319$\pm$0.0244          & 0.5279$\pm$0.0277          & 0.0595$\pm$0.0547          & 0.5662$\pm$0.0244          \\
CLR                 & 0.5167$\pm$0.0583          & 0.4476$\pm$0.1147          & 0.4844$\pm$0.0575          & 0.4795$\pm$0.0583          & -0.0357$\pm$0.1147         & 0.5210$\pm$0.0575          \\
TIGRESS             & 0.1333$\pm$0.0541          & 0.8302$\pm$0.0367          & 0.4585$\pm$0.0374          & 0.2258$\pm$0.0817          & -0.0552$\pm$0.1061         & 0.5567$\pm$0.0358          \\
PCC                 & 0.5042$\pm$0.0124          & 0.4333$\pm$0.0245          & 0.4711$\pm$0.0101          & 0.4661$\pm$0.0124          & -0.0625$\pm$0.0245         & 0.5091$\pm$0.0101          \\
ARACNE              & 0.1167$\pm$0.0519          & \textbf{0.9127$\pm$0.0579} & 0.4881$\pm$0.0197          & 0.2051$\pm$0.0519          & 0.0479$\pm$0.0579          & 0.5808$\pm$0.0197          \\
MINET               & 0.5764$\pm$0.0425          & 0.5381$\pm$0.0888          & 0.5585$\pm$0.0458          & 0.5547$\pm$0.0425          & 0.1147$\pm$0.0888          & 0.5910$\pm$0.0458          \\ \hline
\end{tabular}
\end{scriptsize}
\caption{Performance evaluation of our method and six other state-of-art methods}
\label{tab:performance_proj_method}
\end{table*}

\section{Conclusion}
This paper considers the TVCS constrained optimization, motivated by the intrinsic restrictions for many important applications, for example, in bioinformatics, recommendation system, and crowdsourcing. To solve the cardinality constrained problem, the key step is the projection onto the cardinality constraints. Although the projection onto the overlapped cardinality constraints is NP-hard in general, we prove that if the TVCS condition is satisfied the projection can be reduced to a linear programming. We further prove that there is an iterative algorithm which finds an integer solution to the linear programming within time complexity $O((p+|\mathcal{G}|)\log_\alpha{1\over R})$, where $R$ is the distance from the initial point to the optimization solution and $\alpha < 1$ is the convergence rate. We finally use synthetic experiments and two interesting applications in bioinformatics and crowdsourcing to validate the proposed TVCS model.

\section*{Acknowledgements}
This project is supported in part by the NSF grant CNS-1548078 and the NEC fellowship.
\newpage
{
\bibliographystyle{abbrvnat}
\bibliography{reference}

\begin{thebibliography}{32}
\providecommand{\natexlab}[1]{#1}
\providecommand{\url}[1]{\texttt{#1}}
\expandafter\ifx\csname urlstyle\endcsname\relax
  \providecommand{\doi}[1]{doi: #1}\else
  \providecommand{\doi}{doi: \begingroup \urlstyle{rm}\Url}\fi

\bibitem[Bach et~al.(2012)Bach, Jenatton, Mairal, Obozinski,
  et~al.]{bach2012structured}
F.~Bach, R.~Jenatton, J.~Mairal, G.~Obozinski, et~al.
\newblock Structured sparsity through convex optimization.
\newblock \emph{Statistical Science}, 27\penalty0 (4):\penalty0 450--468, 2012.

\bibitem[Baraniuk et~al.(2010)Baraniuk, Cevher, Duarte, and
  Hegde]{baraniuk2010model}
R.~G. Baraniuk, V.~Cevher, M.~F. Duarte, and C.~Hegde.
\newblock Model-based compressive sensing.
\newblock \emph{Information Theory, IEEE Transactions on}, 56\penalty0
  (4):\penalty0 1982--2001, 2010.

\bibitem[Candes et~al.(2006)Candes, Romberg, and Tao]{candes2006stable}
E.~J. Candes, J.~K. Romberg, and T.~Tao.
\newblock Stable signal recovery from incomplete and inaccurate measurements.
\newblock \emph{Communications on pure and applied mathematics}, 59\penalty0
  (8):\penalty0 1207--1223, 2006.

\bibitem[Cevher et~al.(2009)Cevher, Indyk, Hegde, and
  Baraniuk]{cevher2009recovery}
V.~Cevher, P.~Indyk, C.~Hegde, and R.~G. Baraniuk.
\newblock Recovery of clustered sparse signals from compressive measurements.
\newblock Technical report, DTIC Document, 2009.

\bibitem[El~Halabi and Cevher(2015)]{el2015totally}
M.~El~Halabi and V.~Cevher.
\newblock A totally unimodular view of structured sparsity.
\newblock In \emph{Proceedings of the Eighteenth International Conference on
  Artificial Intelligence and Statistics}, pages 223--231, 2015.

\bibitem[Faith et~al.(2007)Faith, Hayete, Thaden, Mogno, Wierzbowski, Cottarel,
  Kasif, Collins, and Gardner]{faith2007large}
J.~J. Faith, B.~Hayete, J.~T. Thaden, I.~Mogno, J.~Wierzbowski, G.~Cottarel,
  S.~Kasif, J.~J. Collins, and T.~S. Gardner.
\newblock Large-scale mapping and validation of escherichia coli
  transcriptional regulation from a compendium of expression profiles.
\newblock \emph{PLoS biol}, 5\penalty0 (1):\penalty0 e8, 2007.

\bibitem[Haury et~al.(2012)Haury, Mordelet, Vera-Licona, and
  Vert]{haury2012tigress}
A.-C. Haury, F.~Mordelet, P.~Vera-Licona, and J.-P. Vert.
\newblock Tigress: trustful inference of gene regulation using stability
  selection.
\newblock \emph{BMC systems biology}, 6\penalty0 (1):\penalty0 145, 2012.

\bibitem[Hegde et~al.(2009)Hegde, Duarte, and Cevher]{hegde2009compressive}
C.~Hegde, M.~F. Duarte, and V.~Cevher.
\newblock Compressive sensing recovery of spike trains using a structured
  sparsity model.
\newblock In \emph{SPARS'09-Signal Processing with Adaptive Sparse Structured
  Representations}, 2009.

\bibitem[Hegde et~al.(2015{\natexlab{a}})Hegde, Indyk, and
  Schmidt]{hegde2015approximation}
C.~Hegde, P.~Indyk, and L.~Schmidt.
\newblock Approximation algorithms for model-based compressive sensing.
\newblock \emph{Information Theory, IEEE Transactions on}, 61\penalty0
  (9):\penalty0 5129--5147, 2015{\natexlab{a}}.

\bibitem[Hegde et~al.(2015{\natexlab{b}})Hegde, Indyk, and
  Schmidt]{hegde2015nearly}
C.~Hegde, P.~Indyk, and L.~Schmidt.
\newblock A nearly-linear time framework for graph-structured sparsity.
\newblock In \emph{Proceedings of the 32nd International Conference on Machine
  Learning (ICML-15)}, pages 928--937, 2015{\natexlab{b}}.

\bibitem[Ho et~al.(2013)Ho, Jabbari, and Vaughan]{ho2013adaptive}
C.-J. Ho, S.~Jabbari, and J.~W. Vaughan.
\newblock Adaptive task assignment for crowdsourced classification.
\newblock In \emph{Proceedings of The 30th International Conference on Machine
  Learning}, pages 534--542, 2013.

\bibitem[Hoffman(2003)]{hoffman2003approximate}
A.~J. Hoffman.
\newblock On approximate solutions of systems of linear inequalities.
\newblock In \emph{Selected Papers Of Alan J Hoffman: With Commentary}, pages
  174--176. 2003.

\bibitem[Huynh-Thu et~al.(2010)Huynh-Thu, Irrthum, Wehenkel, and
  Geurts]{irrthum2010inferring}
V.~A. Huynh-Thu, A.~Irrthum, L.~Wehenkel, and P.~Geurts.
\newblock Inferring regulatory networks from expression data using tree-based
  methods.
\newblock \emph{PloS one}, 5\penalty0 (9):\penalty0 e12776, 2010.

\bibitem[Jenatton et~al.(2011)Jenatton, Mairal, Obozinski, and
  Bach]{jenatton2011proximal}
R.~Jenatton, J.~Mairal, G.~Obozinski, and F.~Bach.
\newblock Proximal methods for hierarchical sparse coding.
\newblock \emph{Journal of Machine Learning Research}, 12\penalty0
  (Jul):\penalty0 2297--2334, 2011.

\bibitem[Kong et~al.(2014)Kong, Fujimaki, Liu, Nie, and
  Ding]{kong2014exclusive}
D.~Kong, R.~Fujimaki, J.~Liu, F.~Nie, and C.~Ding.
\newblock Exclusive feature learning on arbitrary structures via
  $l_{1,2}$-norm.
\newblock In \emph{Advances in Neural Information Processing Systems}, pages
  1655--1663, 2014.

\bibitem[Liu and Wright(2015)]{liu2015asynchronous}
J.~Liu and S.~J. Wright.
\newblock Asynchronous stochastic coordinate descent: Parallelism and
  convergence properties.
\newblock \emph{SIAM Journal on Optimization}, 25\penalty0 (1):\penalty0
  351--376, 2015.

\bibitem[Marbach et~al.(2009)Marbach, Schaffter, Mattiussi, and
  Floreano]{marbach2009generating}
D.~Marbach, T.~Schaffter, C.~Mattiussi, and D.~Floreano.
\newblock Generating realistic in silico gene networks for performance
  assessment of reverse engineering methods.
\newblock \emph{Journal of computational biology}, 16\penalty0 (2):\penalty0
  229--239, 2009.

\bibitem[Margolin et~al.(2006)Margolin, Nemenman, Basso, Wiggins, Stolovitzky,
  Favera, and Califano]{margolin2006aracne}
A.~A. Margolin, I.~Nemenman, K.~Basso, C.~Wiggins, G.~Stolovitzky, R.~D.
  Favera, and A.~Califano.
\newblock Aracne: an algorithm for the reconstruction of gene regulatory
  networks in a mammalian cellular context.
\newblock \emph{BMC bioinformatics}, 7\penalty0 (Suppl 1):\penalty0 S7, 2006.

\bibitem[Meyer et~al.(2008)Meyer, Lafitte, and Bontempi]{meyer2008minet}
P.~E. Meyer, F.~Lafitte, and G.~Bontempi.
\newblock minet: Ar/bioconductor package for inferring large transcriptional
  networks using mutual information.
\newblock \emph{BMC bioinformatics}, 9\penalty0 (1):\penalty0 461, 2008.

\bibitem[Needell and Tropp(2009)]{needell2009cosamp}
D.~Needell and J.~A. Tropp.
\newblock Cosamp: Iterative signal recovery from incomplete and inaccurate
  samples.
\newblock \emph{Applied and Computational Harmonic Analysis}, 26\penalty0
  (3):\penalty0 301--321, 2009.

\bibitem[Nguyen et~al.(2012)Nguyen, Chin, and Tran]{gradmp}
N.~Nguyen, S.~Chin, and T.~D. Tran.
\newblock A unified iterative greedy algorithm for sparsity constrained
  optimization.
\newblock 2012.

\bibitem[Nguyen et~al.(2014)Nguyen, Needell, and Woolf]{nguyen2014linear}
N.~Nguyen, D.~Needell, and T.~Woolf.
\newblock Linear convergence of stochastic iterative greedy algorithms with
  sparse constraints.
\newblock \emph{arXiv preprint arXiv:1407.0088}, 2014.

\bibitem[Olshausen and Field(1997)]{olshausen1997sparse}
B.~A. Olshausen and D.~J. Field.
\newblock Sparse coding with an overcomplete basis set: A strategy employed by
  v1?
\newblock \emph{Vision research}, 37\penalty0 (23):\penalty0 3311--3325, 1997.

\bibitem[Papadimitriou and Steiglitz(1982)]{papadimitriou1982combinatorial}
C.~H. Papadimitriou and K.~Steiglitz.
\newblock \emph{Combinatorial optimization: algorithms and complexity}.
\newblock Courier Corporation, 1982.

\bibitem[Schaffter et~al.(2011)Schaffter, Marbach, and
  Floreano]{schaffter2011genenetweaver}
T.~Schaffter, D.~Marbach, and D.~Floreano.
\newblock Genenetweaver: in silico benchmark generation and performance
  profiling of network inference methods.
\newblock \emph{Bioinformatics}, 27\penalty0 (16):\penalty0 2263--2270, 2011.

\bibitem[Yang et~al.(2016)Yang, Huang, Tran, Liu, and Huang]{yangbenefits}
H.~Yang, Y.~Huang, L.~Tran, J.~Liu, and S.~Huang.
\newblock On benefits of selection diversity via bilevel exclusive sparsity.
\newblock In \emph{Computer Vision and Pattern Recognition (CVPR), 2016 IEEE
  Conference on}. IEEE, 2016.

\bibitem[Yuan and Lin(2006)]{yuan2006model}
M.~Yuan and Y.~Lin.
\newblock Model selection and estimation in regression with grouped variables.
\newblock \emph{Journal of the Royal Statistical Society: Series B (Statistical
  Methodology)}, 68\penalty0 (1):\penalty0 49--67, 2006.

\bibitem[Yuan et~al.(2014)Yuan, Li, and Zhang]{yuan2014gradient}
X.~Yuan, P.~Li, and T.~Zhang.
\newblock Gradient hard thresholding pursuit for sparsity-constrained
  optimization.
\newblock In \emph{Proceedings of The 31st International Conference on Machine
  Learning}, pages 127--135, 2014.

\bibitem[Yuan et~al.(2016)Yuan, Li, and Zhang]{yuan2016exact}
X.~Yuan, P.~Li, and T.~Zhang.
\newblock Exact recovery of hard thresholding pursuit.
\newblock In \emph{Advances in Neural Information Processing Systems}, pages
  3558--3566, 2016.

\bibitem[Zhang(2009)]{zhang2009consistency}
T.~Zhang.
\newblock On the consistency of feature selection using greedy least squares
  regression.
\newblock \emph{Journal of Machine Learning Research}, 10\penalty0
  (Mar):\penalty0 555--568, 2009.

\bibitem[Zhou et~al.(2010)Zhou, Jin, and Hoi]{zhou2010exclusive}
Y.~Zhou, R.~Jin, and S.~C. Hoi.
\newblock Exclusive lasso for multi-task feature selection.
\newblock In \emph{AISTATS}, volume~9, pages 988--995, 2010.

\bibitem[Zou and Conzen(2005)]{zou2005new}
M.~Zou and S.~D. Conzen.
\newblock A new dynamic bayesian network (dbn) approach for identifying gene
  regulatory networks from time course microarray data.
\newblock \emph{Bioinformatics}, 21\penalty0 (1):\penalty0 71--79, 2005.

\end{thebibliography}
}

\newpage

\begin{center}
{\Large Supplementary Material}
\end{center}
\setcounter{section}{0}
In the supplementary materials, we include all the proofs for the proposed theorems and the detailed derivations for the formulation of the crowdsourcing task assignment problem.
\section{Proof of Lemma~\ref{thm:proj_ip}}
Firstly we show how to convert the projection problem~\eqref{eq:proj_obj} to a support set selection problem. For any vector $\w$, let vector $\x \in \{0,1\}^p$ indicate the nonzero positions of $\w$, then we can claim that
\begin{equation*}
\|\v - \v_{\x}\|^2 \leq \| \v - \w \|^2
\end{equation*}
where $\v_{\x}$ is a vector having same dimension with $\v$, and it keeps elements at positions where $\x$ has ``1'', and fills zeros at positions where $\x$ has ``0''.
In addition, vector $\w \in \Omega(\mathcal{G}, \s)$ if and only if its support set indicator vector $\x$ satisfies $A\x \leq \s$, given $A$ is defined in~\eqref{eq:A_def}.

So the problem~\eqref{eq:proj_obj} can be converted to integer programming:
\begin{align}
\min_{\x \in \{0,1\}^p} \quad& \|\v - \v_{\x}\|^2 \label{eq:ip_dist} \\
\text{subject to} \quad& A\x \leq \s \notag
\end{align}
and the objective can be further simplified:
\begin{align*}
\|\v - \v_{\x}\|^2 &= \langle \v^2, \1 - \x \rangle \\
& = \|\v\|^2 - \langle \v^2, \x \rangle
\end{align*}
Since $\v$ is constant here, then the problem~\eqref{eq:ip_dist} is equivalent to the ILP~\eqref{eq:proj_ilp}, which means problem~\eqref{eq:proj_obj} is equivalent to ILP~\eqref{eq:proj_ilp}. Then we complete the proof.

\section{Proof of Theorem~\ref{thm:iplp}}
To prove Theorem~\ref{thm:iplp}, we use the concept of totally unimodular matrix.

\begin{definition}\label{def:tu}
{\bf Totally Unimodular (TU) Matrix.} An integer matrix is {\rc TU}, if the determinant of any square submatrices\footnote{Submatrix here is a square smaller matrix obtained by removing certain rows and columns} is in the set $\{-1, 0, 1\}$. 
\end{definition}

{\rc
\begin{proposition}\label{prop:idtu}
If $A$ is TU, then $A^\top$ is TU, and their concatenations with identity matrices (i.e. $[ A, I ], [A^\top, I]^\top$) are still TU.
\end{proposition}
\begin{proof}
Since transposing matrix will not change the determinant, so it is obvious $A^\top$ is TU.\\
Then we prove stacking with identity matrix $I$ preserves TU property. We prove it by induction. Firstly, we show that submatrix with size 1 always has determinant in $\{-1, 0, 1\}$, because any element from $I$ is either 1 or 0. Now consider a submatrix with size $k$ having determinant in $\{-1, 0, 1\}$, then submatrix with size $k+1$ will still have determinant in $\{-1, 0, 1\}$. To show this, we only need to prove that adding a new row/column from $I$ will not change the determinant out of set $\{-1, 0, 1\}$. Since any row/column from $I$ only has one nonzero element ``1'', we can eliminate other elements in the same position by subtracting a multiple of this row/column to other rows/columns. After that, we can remove this row and column, and the determinant can only change its sign. So we know that submatrix with size $k+1$ has determinant in $\{-1, 0, 1\}$ if submatrix with size $k$ has determinant in $\{-1, 0, 1\}$.
\end{proof}
}

\begin{lemma}\label{lem:tuint}
If $A$ is {\rc TU}, $\s$ is an integer vector, then {\rc all} vertices of the following polytope are integer points:
\begin{equation}
\{\x \mid A\x\leq \s, \x \in [0, 1]^p\} \label{eq:tulp}
\end{equation}
\end{lemma}
\begin{proof}
We can have an equivalent form of this polytope:
\begin{equation}
 \left \{ \x \middle| \begin{bmatrix} A  \\ I \end{bmatrix}\x \leq \begin{bmatrix} \s \\ \1 \end{bmatrix}, \x \geq \0 \right \}
\end{equation}
From Proposition~\ref{prop:idtu}, we know that matrix 
$$\begin{bmatrix} A  \\ I \end{bmatrix}$$
is TU if $A$ is TU, so this meets the case in Theorem 13.2~\citep[see][chap 13]{papadimitriou1982combinatorial}. Then we complete the proof.
\end{proof}

\begin{lemma}\label{lem:tu2}
If $C$ is the matrix whose each row is the indicator vector of a group $g\in \mathcal{G}_1 \cup \mathcal{G}_2$ in our TVCS model, then $C$ is a TU matrix.
\end{lemma}
\begin{proof}
Since $C$'s rows are the indicator vectors of groups in $\mathcal{G}$, so $C_{ij}\in\{0,1\}$. From the definition~\ref{def:overlap}, we know that there are at most two ``1''s in each column. For the column which has two ``1''s, the two corresponding groups are from $\mathcal{G}_1$ and $\mathcal{G}_2$ respectively. (Because we know that groups within $\mathcal{G}_1$ or $\mathcal{G}_2$ do not overlap.)

By this way, our matrix $C$ meets the case in Theorem 13.3~\citep[see][chap 13]{papadimitriou1982combinatorial}, and it is a TU matrix.
\end{proof}

\begin{lemma}\label{lem:tu}
{\rc Recall the matrix $A$ in equation~\eqref{eq:A_def}:
\begin{equation*}
A=
\begin{bmatrix}
\1^\top\\
C
\end{bmatrix}
\end{equation*}
If it is constructed from the TVCS model, then $A$ is a TU matrix.}
\end{lemma}
\begin{proof}
From Lemma~\ref{lem:tu2}, we know $C$ is a TU matrix for any $\mathcal{G}$ of our TVCS model. {\rc In other words, any submatrix restricted in $C$ has determinant $-1$, $0$, or $1$. Therefore, we only need to consider the submatrix of $A'$ has overlaps with the first row $\1^\top$. There are only three possible forms of such submatrix $S$. We will show all of their determinants are in $\{-1,0,1\}$.} 
\begin{enumerate}
\item[1)] At least one column of $S$ has a single ``$1$'', so it must appear in the first row $\1^\top$. By exchanging such column with the last column (which can only influence the sign of determinant), we can transform it with form:
\[
\left[
\begin{matrix}
\1^\top & 1\\
\bar{C} & \0
\end{matrix}
\right]
\]
where $\bar{C}$ is any submatrix of $C$. From the matrix determinant property, we have $|S|=\pm|\bar{C}| \in \{-1,0,1\}$. Therefore, submatrix $S$ in such form have determinants in $\{-1,0,1\}$.
\item[2)] {\rc All columns of $S$ have three ``$1$'' elements (the last row has ``1'' for every column). For the rows which are from $C$, we can sum all the rows to a certain row (this will not  change the determinant). By this way we transform $S$ to the following form:}
\[
\left[
\begin{matrix}
\1^\top\\
{\bf 2}^\top\\
\bar{C}
\end{matrix}
\right]
\]
where $\bar{C}$ is a submatrix of $C$. In this case, $S$ is not full rank, so its determinant is 0.
\item[3)] Each column in $S$ contains at least two ``$1$'' elements, and there exists one column which has exactly two ``$1$''s. By exchanging it to the last column, we can transform it to be:
\[
\left[
\begin{matrix}
\1^\top & 1
\\
\cdots & \0
\\
\bar{C}_i & 1
\\
\cdots & \0
\end{matrix}
\right]
\]

This means that one ``1'' is in the first row, and another is in the row from $C$, let us say it's $\bar{C}_i$. Since subtracting one row from another row will not change the determinant, we can let the first row subtract $\bar{C}_i$:
\[
\left[
\begin{matrix}
\1^\top & 1
\\
\cdots & \0
\\
\bar{C}_i & 1
\\
\cdots & \0
\end{matrix}
\right]
\rightarrow
\left[
\begin{matrix}
\1^\top-\bar{C}_i & 0
\\
\cdots & \0
\\
\bar{C}_i & 1
\\
\cdots & \0
\end{matrix}
\right]
\]
Now the last column only has a single ``1'' in the $i$-th row. We can generate a smaller matrix $S'$ by removing the $i$-th row and the last column, and if $S'$ has determinant in $\{-1, 0, 1\}$, so does $S$.

\[
\left[
\begin{matrix}
\1^\top-\bar{C}_i & 0
\\
\cdots & \0
\\
\bar{C}_i & 1
\\
\cdots & \0
\end{matrix}
\right]
\rightarrow
S'=\left[
\begin{matrix}
\1^\top-\bar{C}_i
\\
\cdots
\\
\cdots
\end{matrix}
\right]
\]

If $\bar{C}_i\neq \0^\top$, then there are some positions (including $j$-th column) in the first row will become zeros. For any column of matrix $S'$ which has ``0'' element in the first row, there are two cases:
\begin{enumerate}
\item This column only contains zeros, i.e. $S'$ has zero determinant.
\item This column contains a single ``1'', we can generate a smaller matrix $S''$ by removing this column and the row where this ``1'' sits. If $S''$ has determinant in $\{-1, 0, 1\}$, so does $S'$.
\end{enumerate}
Notice that it is impossible for the case that such column has two ``1''s. (Since each column can have at most three ``1''s, and we already remove the ``1'' in the first row  by subtraction, and discard another ``1'' by removing $\bar{C}_i$.) In the above case (b), we can repeat removing columns and rows until we get a degenerate matrix (has 0 determinant), or a matrix whose first row does not have zeros. For the later situation, we can process it by same procedures as the original matrix $S$ unless it only has one row and one column, i.e. a matrix having single element ``1'' (has determinant 1).

If $\bar{C}_i=\0^\top$, we can also process it by same procedures as the original matrix $S$.
\end{enumerate}
Therefore, we have proved that any square submatrix $S$ in $A'$ has determinant in $\{-1, 0, 1\}$, which means $A'$ is TU, and hence $A$ is TU.
\end{proof}

Applying Lemma~\ref{lem:tuint} and Lemma~\ref{lem:tu}, we complete the proof of Theorem~\ref{thm:iplp}.

\section{Proof of Theorem~\ref{thm:iteralgo}}
To prove Theorem~\ref{thm:iteralgo}, we start with several lemmas.
\begin{lemma}\label{lem:osc}
Formulate the feasibility problem as problem~\eqref{eq:lpgd}, let $f$ be the objective of the formulation in Theorem~\ref{thm:iteralgo}. If $f^* = 0$, there exists a $\lambda$ such that
\begin{equation*}
f(\z) - f^* \geq \frac{\lambda}{2} \|\z - P_{\z^*}(\z)\|^2, \forall \z \in \Omega
\end{equation*}
where $P_{\z^*}(\z)$ is the optimal point which is closet to $\z$.
\end{lemma}

\begin{proof}
Since $f^* = 0$, there exists at least an $\z^*$ such that
\begin{align*}
A\z^* - a \leq 0\\
B\z^* = b\\
C\z^* \leq c
\end{align*}
From Hoffman's Theorem~\citep{hoffman2003approximate}, we know that there exists a $\lambda > 0$, such that
\begin{equation*}
\frac{\lambda}{2} \|\z - P_{\z^*}(\z)\|^2 
\leq \|[A\z-a]_+ \|^2 + \|B\z-b\|^2 + \|[C\z-c]_+ \|^2
\end{equation*}
Therefore, we know for any $\z$ in $\Omega$,
\begin{equation*}
C\z \leq c
\end{equation*}
and
\begin{equation*}
\frac{\lambda}{2} \|\z - P_{\z^*}(\z)\|^2 
\leq \|[A\z-a]_+ \|^2 + \|B\z-b\|^2
\end{equation*}
\end{proof}

Using the lemma above, we can prove the Theorem~\ref{thm:iteralgo} now.

\begin{proof}
Denote by $\Delta_{t}:= \|\z^{t} - P_{\z^*}(\z^{t})\|^2$. We have
\begin{align*}
\Delta_{t+1} &= \|\z^{t+1} - P_{\z^*}(\z^{t+1})\|^2 \\ 
& \leq \|\z^{t+1} - P_{\z^*}(\z^{t})\|^2 \\
& \leq \|\z^{t} - P_{\z^*}(\z^{t}) + \z^{t+1} - \z^t\|^2 \\
& = \|\z^{t} - P_{\z^*}(\z^{t})\|^2 + \|\z^{t+1} - \z_t\|^2 + 2\langle \z^{t+1} - \z^t, \z^t - P_{\z^*}(\z^{t})\rangle\\
& = \Delta_t - \|\z^{t+1} - \z^t\|^2 + 2\langle \z^{t+1} - \z^t, \z^{t+1} - P_{\z^*}(\z^{t})\rangle\\
& \leq \Delta_t - \|\z^{t+1} - \z^t\|^2 - 2\gamma\langle \nabla f(\z^t), \z^{t+1} - P_{\z^*}(\z^{t})\rangle
\end{align*}
Let $T = \langle \nabla f(\z^t), \z^{t+1} - P_{\z^*}(\z^{t})\rangle$. Then we have
\begin{align*}
T 
& = \langle \nabla f(\z^t), \z^{t+1} - P_{\z^*}(\z^{t})\rangle \\
& = \langle \nabla f(\z^t), \z^{t} - P_{\z^*}(\z^t)\rangle + \langle \nabla f(\z^t), \z^{t+1} - \z^t\rangle \\
& \geq -f^* + f(\z^t) + f(\z^{t+1}) - f(\z^t) - \frac{L}{2}\|\z^{t+1} - \z^t\|^2 \\
& = f(\z^{t+1}) - f^* - \frac{L}{2}\|\z^{t+1} - \z^t\|^2
\end{align*}
where $L$ is the Lipschitz continuous gradient constant. Back to the original inequality, we have
\begin{align*}
\|\z^{t+1} - P_{\z^*}(\z^{t+1})\|^2 
& \leq \Delta_t - \|\z^{t+1} - \z^t\|^2 - 2\gamma\langle \nabla f(\z^t), \z^{t+1} - P_{\z^*}(\z^{t})\rangle \\
& \leq \Delta_t - (1-L\gamma) \|\z^{t+1} - \z^t\|^2 - 2\gamma(f(\z^{t+1}) - f^*) \\
& \leq \Delta_t - (1-L\gamma) \|\z^{t+1} - \z^t\|^2 - 2\gamma\frac{\lambda}{2} \|\z^{t+1} - P_{\z^*}(\z^{t+1})\|^2
\end{align*}
where the last inequality comes from Lemma~\ref{lem:osc}.

Let $\gamma = \frac{1}{L}$ and we have
\begin{align*}
(1 + \frac{\lambda}{L})\Delta_{t+1} & \leq \Delta_t \\
\Delta_{t+1} & \leq \frac{1}{1+\frac{\lambda}{L}}\Delta_t
\end{align*}
which shows the linear convergence rate $\alpha=\frac{1}{1+\frac{\lambda}{L}}$, then it completes the proof.
\end{proof}

\section{Formulation of the Expected Accuracy in Crowdsourcing Task Assignment}
In crowdsourcing task assignment problem, recall the objective function of problem~\eqref{eq:crowd}:
\begin{equation*}
{1\over m}\sum_{j=1}^{m}\mathcal{E}_{\text{acc}}(Q_{\cdot,j}, X_{\cdot,j})
\end{equation*}
%
For the $j$-th task, $\mathcal{E}_{\text{acc}}(Q_{\cdot,j}, X_{\cdot,j})$ is defined in the following:
\begin{align}
\mathcal{E}_{\text{acc}}(Q_{\cdot,j}, X_{\cdot,j})
=& \P(\hat{\y}_j=1 | \y_j=1 )\P(\y_j=1) + \P(\hat{\y}_j=0 | \y_j=0 )\P(\y_j=0) \notag\\
=&\E_{\hat{Y}_{\Omega_j,j}|\y_j=1}[{\bf I}(\hat{\y}_j=1)]\P(\y_j=1) +\E_{\hat{Y}_{\Omega_j,j}|\y_j=0}[{\bf I}(\hat{\y}_j=0)]\P(\y_j=0) \label{eq:expacc}
\end{align}
where ${\bf I}(\cdot)$ is the indicator function. We can further {\rc specify} this formulation by considering the equivalent forms for $\hat{\y}_j=1$ and $\hat{\y}_j =0$:
\begin{align*}
&\hat{\y}_j=1 \\
\Leftrightarrow & \P(\y_j=1 | \hat{Y}_{\Omega_j,j}) \geq \P(\y_j=0 | \hat{Y}_{\Omega_j,j}) \\
\Leftrightarrow & {\P(\hat{Y}_{\Omega_j,j} | \y_j=1) \over \P(\hat{Y}_{\Omega_j,j} | \y_j=0)} \geq {\P(\y_j=0) \over \P(\y_j=1)}\\
\Leftrightarrow & \prod_{i\in \Omega_j} \left({Q_{ij} \over 1-Q_{ij}}\right)^{2\hat{Y}_{i,j}-1} \geq {\P(\y_j=0) \over \P(\y_j=1)}\\
\Leftrightarrow & \sum_{i\in \Omega_j} (2\hat{Y}_{i,j}-1) \log\left({Q_{ij} \over 1-Q_{ij}}\right) \geq \log\left({\P(\y_j=0) \over \P(\y_j=1)}\right)\\
\Leftrightarrow & \sum_{i = 1}^n X_{ij}(2\hat{Y}_{i,j}-1) \log\left({Q_{ij} \over 1-Q_{ij}}\right) - \log\left({\P(\y_j=0) \over \P(\y_j=1)}\right) \geq 0
\end{align*}
Similar derivation can be applied to $\hat{\y}_j=0$ (change ``$\geq$'' to ``$<$''). Here we substitute the indicator function ${\bf I}(t\geq0)$ as sigmoid function $S(t)={1\over 1+\exp(-t)}$ {\rc to obtain a smooth approximation. Denote by $Z_{ij}:=(2\hat{Y}_{i,j}-1) \log\left({Q_{ij} / (1-Q_{ij}})\right)$ and $\r_j:=\log\left({\P(\y_j=0) / \P(\y_j=1)}\right)$ for short. The (smooth) objective turns out to be}:
\begin{equation*}
 {1\over m}\sum_{j=1}^{m} \Bigg(\E_{\hat{Y}|\y_j=1} \left [S\left(\sum_{i = 1}^n Z_{ij}X_{ij}-\r_j\right)\right]\P(\y_j=1) +\E_{\hat{Y}|\y_j=0}  \left [S\left(\r_j - \sum_{i = 1}^n Z_{ij}X_{ij}\right)\right]\P(\y_j=0)\Bigg)
\end{equation*}
and its stochastic gradient is:
\begin{align*}
\g(X)_{\cdot, j} = &{1\over m} \P(\y_j = 1) (1-S(\sum_{i = 1}^n Z_{ij}^{\y_j = 1} X_{ij}-\r_j))S(\sum_{i = 1}^n Z_{ij}^{\y_j = 1}X_{ij}-\r_j) Z_{\cdot,j}^{\y_j = 1} +\\
& {1\over m}\P(\y_j = 0) (1-S(\sum_{i = 1}^n Z_{ij}^{\y_j = 0} X_{ij}-\r_j))S(\sum_{i = 1}^n Z_{ij}^{\y_j = 0}X_{ij}-\r_j) Z_{\cdot,j}^{\y_j = 0}
\end{align*}
where $Z^{\y_j = 1}$ (or $Z^{\y_j = 0}$) is generated by sampling $\hat{Y}$ given $\y_j=1$ (or $\y_j=0$).
\end{document}